\documentclass{article}

\usepackage{multicol}
\usepackage{multirow}
\usepackage{dsfont}

\usepackage[textsize=tiny]{todonotes}
\usepackage{amsmath}
\usepackage{mathtools}
\usepackage{graphicx}
\usepackage{times}
\usepackage{helvet}
\usepackage{courier}
\usepackage{paralist}
\usepackage{latexsym}
\usepackage{url}
\usepackage[all]{xy}
\usepackage{amsmath}
\usepackage{amssymb}
\usepackage{amsthm}
\usepackage{nccmath} 
\usepackage{comment}
\usepackage{paralist}
\usepackage{xcolor}
\usepackage{graphicx}
\usepackage{pifont}
\usepackage{savesym}
\savesymbol{AND}
\usepackage{xspace}
\usepackage{tikz}
\usepackage{pgfplots}
\usepackage{pgf}
\usepackage{algorithm}
\usepackage{algorithmic}
\usepackage{xspace}
\usepackage{comment}
\usepackage{placeins}

\newcommand{\rsa}{\rightsquigarrow}

\newcommand{\nn}{\nonumber}

\newcommand{\R}{\Re} 

\newcommand{\ra}{\rightarrow}

\newcommand{\E}[1]{\mathbb{E}\left[#1\right]}

\newcommand{\N}{\mathcal{N}}

\newcommand{\Tb}{{\Theta}}

\newcommand{\Tw}{{\Theta}^w}
\newcommand{\Tg}{{\Theta}^g}

\newcommand{\tg}{\theta^g}

\newcommand{\G}{\mathcal{G}}

\newcommand{\norm}[1]{\|#1\|}

\renewcommand{\epsilon}{\varepsilon}

\newtheorem{theorem}{Theorem}[section]
\newtheorem{lemma}[theorem]{Lemma}

\newtheorem{corollary}[theorem]{Corollary}
\newtheorem{assumption}{Assumption}
\newtheorem{definition}{Definition}[section]
\newtheorem{remark}{Remark}[section]

\newcommand{\ip}[1]{\langle #1\rangle}

\def\Re{\mathbb{R}}

\def\P{\mathcal{P}}
\def\S{\mathcal{S}}
\def\A{\mathcal{A}}

\newcounter{subequation}[equation]

\def\mathdisplay#1{%
  \ifmmode \@badmath
  \else
    $\def\@currenvir{#1}%
    \let\dspbrk@context\z@
    \let\tag\tag@in@display \SK@equationtrue 
    \global\let\df@label\@empty \global\let\df@tag\@empty
    \global\tag@false
    \let\mathdisplay@push\mathdisplay@@push
    \let\mathdisplay@pop\mathdisplay@@pop
    \if@fleqn
      \edef\restore@hfuzz{\hfuzz\the\hfuzz\relax}%
      \hfuzz\maxdimen
      \setbox\z@\hbox to\displaywidth\bgroup
        \let\split@warning\relax \restore@hfuzz
        \everymath\@emptytoks \m@th $\displaystyle
    \fi
}

\newcounter{algostep}

\newcounter{acalgorithm}

\usepackage{comment}
\usepackage{cancel}
\usepackage{changepage}
\usepackage{bbm}

\usepackage{pgfplots}
\usepackage{filecontents}

\usepackage{placeins}

\usepackage{microtype}
\usepackage{graphicx}
\usepackage{subfigure}
\usepackage{booktabs} 

\usepackage{hyperref}
\usepackage[capitalize]{cleveref}
\usepackage{caption}
\usepackage{wrapfig,lipsum}
\usepackage[linewidth=1.2pt,linecolor=red]{mdframed}

\usepackage[accepted]{icml2019}

\icmltitlerunning{Deep Gated Networks (A study In Paths)}
\title{Deep Gated Networks: A framework to understand training and generalisation in deep learning}
\author{Chandrashekar Lakshminarayanan${}^*$ and Amit Vikram Singh\thanks{Both authors contributed equally.},\\ Indian Institute of Technology Palakkad\\\{chandru@iitpkd.ac.in,amitkvikram@gmail.com\}}
\begin{document}

\maketitle
\begin{abstract}
Understanding the role of (stochastic) gradient descent (SGD) in the training and generalisation of deep neural networks (DNNs) with ReLU activation has been the object study in the recent past. In this paper, we make use of deep gated networks (DGNs) as a framework to obtain insights about DNNs with ReLU activation. In DGNs, a single neuronal unit has two components namely the pre-activation input (equal to the inner product the weights of the layer and the previous layer outputs), and a gating value which belongs to $[0,1]$ and the output of the neuronal unit is equal to the multiplication of pre-activation input and the gating value. The standard DNN with ReLU activation, is a special case of the DGNs, wherein the gating value is $1/0$ based on whether or not the pre-activation input is positive or negative. We theoretically analyse and experiment with several variants of DGNs, each variant suited to understand a particular aspect of either training or generalisation in DNNs with ReLU activation.
Our theory throws light on two questions namely i) why increasing depth till a point helps in training and ii) why increasing depth beyond a point hurts training? We also present experimental evidence to show that gate adaptation, i.e., the change of gating value through the course of training is key for generalisation.
\end{abstract}
\section{Introduction}\label{sec:intro}
Given a dataset $(x_s,y_s)_{s=1}^n\in \R^{d_{in}}\times \R$, and a deep neural network (DNN), parameterised by $\Theta\in \R^{d_{net}}$, whose prediction on input example $x\in \R^{d_{in}}$ is $\hat{y}_{\Theta}(x_s)\in \R$, in this paper, we are interested in the stochastic gradient descent (SGD) procedure to minimise the squared loss given by $L_{\Theta}=\sum_{s=1}^n \left(\hat{y}_{\Tb}(x_s) - y_s\right)^2$.  
As with some of the recent works by \citet{dudnn,dudln} to understand SGD in deep networks, we adopt the trajectory based analysis, wherein, one looks at the (error) \emph{trajectory}, i.e., the dynamics of the error defined as $e_t(s)\stackrel{def}=\hat{y}_{\Tb_t}(x_s)-y_s$. Let $e_t\stackrel{def}=(e_t(s),s\in[n])\in\R^n$\footnote{We use $[n]$ to denote the set $\{1,\ldots,n\}$.}, then the error dynamics is given by:
\begin{align}\label{eq:basictraj}
e_{t+1}=e_t-\alpha_t K_t e_t,
\end{align}
where $\alpha_t>0$ is a small enough step-size, $K_t=\Psi_t^\top\Psi_t$ is an $n\times n$ \emph{Gram} matrix, and $\Psi_t$ is a $d_{net}\times n$ neural tangent feature (NTF) matrix whose entries are given by $\Psi_t(m,s)=\frac{\partial \hat{y}_{\Theta_t}(x_s)}{\partial \theta(m)}$\footnote{We assume that the weights can be enumerated as $\theta(m),m=1,\ldots, d_{net}$.}. In particular, we obtain several new insights related to the following:

$1.$ \emph{The Depth Phenomena}: It is well known in practice that increasing depth (of DNNs) till a point improves their training performance. However, increasing the depth beyond a point degrades training. We look at the spectral properties of the Gram matrix $K_0$ for randomised (symmetric Bernoulli) initialisation, and reason about the depth phenomena.

$2.$ \emph{Gate adaptation}, i.e., the dynamics of the gates in a deep network and its role in generalisation performance.

\textbf{Conceptual Novelties:} In this paper, we bring in two important conceptual novelties. First novelty is the framework of \emph{deep gated networks} (DGN), previously studied by \cite{sss}, wherein, the gating is decoupled from the pre-activation input. 
 Second novelty is what we call as the \emph{path-view}. We describe these two concepts first, and then explain the gains/insights we obtain from them.

\textbf{Deep Gated Networks (DGNs):} We consider networks with depth $d$, and width $w$ (which is the same across layers). At time $t$, the output $\hat{y}_{t}(x)\in \R$ of a DGN for an input $x\in \R^{d_{in}}$ can be specified by its gating values and network weights $\Theta_t\in \R^{d_{net}}$ as shown in \Cref{tb:dgn}.

\FloatBarrier
\begin{table}[h]
\centering
\begin{tabular}{|c|c|}\hline
Input layer & $z_{x_s,\Theta_t}(0)=x_s$ \\\hline
Pre-activation & $q_{x_s,\Theta_t}(l)={\Theta_t(l)}^\top z_{x_s,\Theta_t}(l-1)$\\\hline
Layer output & $z_{x_s,\Theta_t}(l)=q_{x_s,\Theta_t}(l)\odot G_{x_s,t}(l)$ \\\hline
Final output & $\hat{y}_t(x_s)={\Theta_t(d)}^\top z_{x_s,\Theta_t}(d-1)$\\\hline
\end{tabular}
\caption{A deep gated network. Here $x_s\in \R^{d_{in}},s\in [n]$ is the input, and $l\in[d-1]$ are the intermediate layers. $G_{x_s,t}(l)\in [0,1]^w$ and $q_{x,\Theta_t}(l)\in \R^w$ are the gating and pre-activation input values respectively at time $t$.}
\label{tb:dgn}
\end{table}
\newpage
$\Theta_t$ together with the collection of the gating values at time $t$ given by $\G_t\stackrel{def}=\{G_{x_{s},t}(l,i) \in [0,1], \forall s\in[n],l\in[d-1],i\in[w]\}$ (where $G(l,i)$ is the gating of $i^{th}$ node in $l^{th}$ layer), recovers the outputs $\hat{y}_t(x_s)\in \R$ for all the inputs $\{x_s\}_{s=1}^n$ in the dataset using the definition in \Cref{tb:dgn}. 

Note that the standard DNN with ReLU activation is a special DGN, wherein, $G_{x_s,t}(l,i)$, the $i^{th}$ node in the $l^{th}$ layer is given by $G_{x_s,t}(l,i)=\mathbbm{1}_{\{q_{x_s,\Theta_t}(l,i) >0\}}$.	

\textbf{Path-View:} A \emph{path} starts from an input node $i\in[d_{in}]$ of the given network, passes through any one of the weights in each layer of the $d$ layers and ends at the output node. Using the paths, we can express the output as the summation of individual  path contributions. The path-view has two important gains: i)  since it avoids the usual layer after layer approach we are able to obtain explicit expression for information propagation that separates the `signal' (the input $x_s\in \R^{d_{in}},s\in[n]$) from the `wire' (the connection of the weights in the network)  (ii) the role of the sub-networks becomes evident. Let $x\in \R^{d_{in}\times n}$ denote the data matrix, and let $\Theta_t(l,i,j)$ denote the ${(i,j)}^{th}$ weight in the $l^{th}$ layer and let $\P=[d_{in}]\times [w]^{d-1}\times[1]$ be a cross product of index sets. Formally,

$\bullet$ A path $p$ can be defined as $p\stackrel{def}=(p(0),p(1),\ldots,p(d))\in \P$, where $p(0)\in [d_{in}]$, $p(l)\in[w],\,\forall l\in[d-1]$ and $p(d)=1$. We assume that the paths can be enumerated as $p=1,\ldots, P = d_{in}w^{d-1}$. Thus, throughout the paper, we use the symbol $p$ to denote a path as well as its index in the enumeration.

$\bullet$ The \emph{strength} of a path $p$ is defined as $w_t(p)\stackrel{def}=\Pi_{l=1}^d \Theta_t(l,p(l-1),p(l))$. 

$\bullet$ The \emph{activation} level of a path $p$ for an input $x_s\in \R^{d_{in}}$ is defined as $A_{\G_t}(x_s,p)\stackrel{def}{=}\Pi_{l=1}^{d-1} G_{x_s,t}(l,p(l))$.

\textbf{Conceptual Gain I (Feature Decomposition):} Define $\phi_{x_s,\G_t}\in \R^P$, where $\phi_{x_s,\G_t}(p)\stackrel{def}=x(p(0),s)A_{\G_t}(x_s,p)$. The output is then given by:
\begin{align}\label{eq:featstrength}
\hat{y}_{t}(x_s)=\phi_{x_s,\G_t }^\top w_{t},
\end{align}	
where $w_{t}=(w_{t}(p),p=1,\ldots,P)\in \R^P$. In this paper, we interpret $\phi_{x_s,\G_t}\in \R^P$ as the \emph{hidden feature} and $w_{t}\in \R^P$, the strength of the paths as the \emph{weight vector}.

\textbf{A hard dataset for DNNs:} The ability of DNNs to fit data has been demonstrated in the past \cite{ben}, i.e., they can fit even random labels, and random pixels of standard datasets such as MNIST. 
However, for standard DNNs with ReLU gates, with no bias parameters, a dataset with $n=2$ points namely $(x,1)$ and $(2x,-1)$ for some $x\in \R^{d_{in}}$ cannot be memorised. The reason is that the gating values are the same for both $x$ and $2x$ (for that matter any positive scaling of $x$), and hence $\phi_{2x,\G_t }= 2\phi_{x,\G_t }$, and thus it not possible to fit arbitrary values for $\hat{y}_t(x)$ and $\hat{y}_t(2x)$.

\textbf{Conceptual Gain II (Similarity Metric):}  In DGNs similarity of two different inputs $x_s,x_{s'}\in \R^{d_{in}}, s,s' \in [n]$ depends on the overlap of sub-networks that are simultaneously \emph{active} for both the inputs. Let $\Phi_t=\left[\phi_{x_1,\G_t},\ldots,\phi_{x_n,\G_t}\right]\in\R^{P\times n}$ be the hidden feature matrix obtained by stacking $\phi_{x_s,t},\forall s\in[n]$. Now the Gram matrix $M_t$ of hidden features is given by $M_t=\Phi^\top_t\Phi_t=(x^\top x)\odot \lambda_t$ where $\lambda_t(s,s')\stackrel{def}=\underset{p\rsa i}{\sum} A_{\G_t}(x_s,p) A_{\G_t}(x_{s'},p)$\footnote{Here $p\rsa (\cdot)$ denote the fact that a path $p$ passes through $(\cdot)$ (which is either a node or a weight). },  stands for the total number of paths that start at any input node $i$ (due to symmetry this number does not vary with $i\in [d_{in}]$) and are \emph{active} for both input examples $s,s'\in[n]$. Each input example has a sub-network that is \emph{active}, and similarity (inner product) of two different inputs depends on the similarity of between the corresponding sub-networks (in particular the total number of paths that are simultaneously active) that are active for the two inputs.

\textbf{Conceptual Gain III (Deep Information Propagation):} An explicit expression for the Gram matrix as $K_t(s,s')=\sum_{i=1}^{d_{in}} x_s(i) x_{s'}(i)\kappa_t(s,s',i), \forall s,s'\in[n]$. Here, $\kappa_t(s,s',i)\in \R$ is a summation of the inner-products of the \emph{path features} (see \Cref{sec:optimisation}). Thus the input signals $x_s,x_{s'}\in \R^{d_{in}}$ stay as it is in the calculations in an algebraic sense, and are separated out from the wires, i.e., the network whose effect is captured in $\kappa_t(s,s,',i)$.

\textbf{A Decoupling assumption:} We assume that the gating $\G_0$ and weights $\Theta_0$ are statistically independent, and that weights ($d_{net}$ of them ) are sampled from $\{-\sigma,+\sigma\}$ with probability $\frac{1}2$. Under these assumptions we obtain the following key results and insights:

$1.$ \textbf{Depth Phenomena I:} Why does increasing depth till a point helps training?

Because, \emph{increasing depth causes whitening of inputs}. In particular, we show that $\E{K_0}=d\sigma^{2(d-1)}\left(x^\top x \odot \lambda_0\right)$, where $\odot$ is the \emph{Hadamard} product. The ratio $\frac{{\lambda_0}(s,s')}{{\lambda_0}(s,s)}$ is the fractional overlap of active sub-networks, say at each layer the overlap of active gates is $\mu\in (0,1)$, then for a depth $d$, the fractional overlap decays at exponential rate, i.e., $\frac{{\lambda_0}(s,s')}{{\lambda_0}(s,s)}\leq \mu^d$, leading to whitening.

$2.$ \textbf{Depth Phenomena II:} Why does increasing depth beyond a point hurts training?

Because, $Var\left[K_0(x,x')\right]\leq O(\max\{\frac{d^2}{w}, \frac{d^3}{w^2}\})$ (for $\sigma^2=O(\frac{1}w)$). Thus for large width $K_0$ converges to its expected value. However, for a fixed width, increasing depth makes the entries of $K_0$ deviate from $\E{K_0}$, thus degrading the spectrum of $K_0$.

$3.$ \textbf{Key Take away:}  To the best of our knowledge, we are the first to present a theory to explain the depth phenomena. While the ReLU gates do not satisfy the decoupling assumption, we hope to relax the decoupling assumption in the future and extend the results for decoupled gating to ReLU activation as well.

\textbf{Conceptual Gain IV  (Twin Gradient Flow):} The NTF matrix can be decomposed as
$
\Psi_t(m,s)=\underbrace{\phi_{x_s,\G_t }^\top\frac{\partial w_{t}} {\partial \theta(m)}}_{\text{strength adaptation}}+ \underbrace{\frac{\partial \phi_{x_s,\G_t }^\top}{\partial \theta(m)} w_{t} }_{\text{gate adaptation}}
$, from which it is evident that the gradient has two components namely i) \emph{strength adaptation:} keeping the sub-networks (at time $t$) corresponding to each input example fixed, the component learns the strengths of the paths in those sub-networks, and  ii) \emph{gate adaptation:} this component learns the sub-networks themselves.  Ours is the first work to analytically capture the two gradients.

\textbf{Conceptual Gain V  (Fixing the ReLU artefact):} In standard ReLU networks the gates are $0/1$ and hence $\frac{\partial \phi_{x_s,\G_t }^\top}{\partial \theta(m)}=0$. Thus the role of gates has been unaccounted for in the current literature.
By parameterising the gates by $\Tg\in\R^{d_{net}}$, and introducing a soft-ReLU gating (with values in $(0,1)$ ), we can show that Gram matrix can be decomposed into $K_t=K^w_t+K^a_t$, where $K^w_t$ is the Gram matrix of strength adaptation and $K^a_t$ is the Gram matrix corresponding to activation adaptation. Ours is the first work to point out that the Gram matrix has a gate adaptation component.

$\bullet$ \textbf{Conceptual Gain VI (Sensitivity sub-network) :} Our expression for $\E{K_0^a}$ involves what we call \emph{sensitivity} sub-network formed by gates which take intermediate values, i..e, are not close to either $0$ or $1$. We contend that by controlling such sensitive gates, the DGN is able to learn the features $\phi_{x_s,\G_t}$ over the course of training.

$\bullet$ \textbf{Evidence I (Generalisation needs gate adaptation):}  We experiment with two datasets namely standard CIFAR-10 (classification) and Binary-MNIST which has classes $4$ and $7$ with labels $\{-1,+1\}$ (squared loss). We observe that whenever gates adapt, test performance gets better.

$\bullet$ \textbf{Evidence II (Lottery is in the gates:)} We obtain $56\%$ test accuracy just by tuning the gates of a parameterised DGN with soft-gates. We also observe that by copying the gates from a learnt network and training the weights from scratch also gives good generalisation performance. This gives a new interpretation for the lottery ticket hypothesis \cite{lottery}, i.e., the real lottery is in the gates.

\textbf{Lessons Learnt:} Rethinking generalisation needs to involve a study on how gates adapt. Taking a cue from \cite{arora}, we look at $\nu_t=y^\top K_t^{-1}y$, and observe that $\nu_t$ in the case of adaptive gates/learned gates is always upper bounded by $\nu_t$ when gates are non-adaptive/non-learned gates.

\textbf{Organisation:} The rest of the paper has \Cref{sec:optimisation}, where, we consider DGNs with fixed or frozen gates, and \Cref{sec:generalisation}, where, we look at DGNs with adaptable gates. The idea is to obtain insights by progressing stage by stage from easy to difficult variants of DGNs, ordered naturally according to the way in which paths/sub-networks are formed. The proofs of the results are in the supplementary material.

\section{Deep Information Propagation in DGN}\label{sec:optimisation}
In this section, we study deep information propagation (DIP) in DGNs when the gates are frozen, i.e., $\G_t=\G_0,\forall t\geq 0$, and our results are applicable to the following:

(i) \emph{Deep Linear Networks (DLN):} where, all the gating values are $1$. Here, since all the paths are \emph{always on}, we do not have any control over how the paths are formed.

(ii) \emph{Fixed random gating (DGN-FRG):} Note that $\G_0$ contains $n\times (d-1)\times w$ gating values, corresponding to the $n$ input examples, $(d-1)$ layer outputs and $w$ nodes in each layer. In DGN-FRG, $\G_0\in \{0,1\}^{n\times (d-1)\times w}$, where each gating value is chosen to be $0$ or $1$ with probabilities $1-p$ and $p$ respectively. Here, we have full control over the gating/activation level of the paths. These networks are restricted solely towards understanding questions related to optimisation. Generalisation is irrelevant for DGN-FRG networks because there is no natural means to extend the random gate generation for unseen inputs.

(iii) \emph{Gated linear unit (GaLU):} networks, wherein the gating values are generated by \emph{another separate} network which is DNN with ReLU gates. Unlike DGN-FRG, GaLU networks can generalise.

We first express the Gram matrix $K_t$ in the language of the paths. We then state our assumption in Assumption~\ref{assmp:mainone},~\ref{assmp:maintwo} followed by our main result on deep information propagation in DGNs (\Cref{th:dgnexp}). We then demonstrate how our main result applies to DGN-FRG and GaLU networks.\footnote{DLN discussion has been moved to the Supplementary Material in the end.}

When the gates are frozen, the weight update affects only the path strengths $w_t(p),p\in[P]$. This is captured as follows:

\textbf{Sensitivity of the path strength}: Let $p\in[P]$ be a path, and $w_t(p)$ be its strength at time $t$. Further, let $\theta(m),m\in[d_{net}]$ be a weight and without loss of generality, let $\theta(m)$ belong to layer $l'(m)\in[d]$. At time $t$, the derivative of path $p$ with respect to $\theta(m)$ denoted by $\varphi_{t,p}(m)$ is given by:
\begin{align}
\begin{split}
\varphi_{t,p}(m)&=\underset{l\neq l'(m)}{\underset{l=1}{\overset{d}{\Pi}}} \Theta_t(l,p(l-1),p(l)), \forall p\rsa\theta(m),\\
\varphi_{t,p}(m)&=0, \forall p\bcancel{\rsa}\theta(m)
\end{split}
\end{align}
The sensitivity of a path $p$ with respect to $\Theta_t\in\R^{d_{net}}$ is then given by the vector $\varphi_{t,p}=(\varphi_{t,p}(m),m\in[d_{net}])\in \R^{d_{net}}$.

\begin{lemma}[Signal vs Wire Decomposition]\label{lm:sigwire}
Let $\kappa_t(s,s',i)\stackrel{def}=\underset{p_1,p_2\rsa i}{\sum_{p_1,p_2\in P:}} A_{\G_t}(x_s,p_1) A_{\G_t}(x_{s'},p_2) \ip{\varphi_{t,p_1}, \varphi_{t,p_2}}$. The Gram matrix $K_t$ is then given by 
\begin{align}\label{eq:ktalg}
{K_t(s,s')}=\sum_{i=1}^{d_{in}} x(i,s)x(i,s') \kappa_t(s,s',i)
\end{align}
\end{lemma}
In Lemma~\ref{lm:sigwire}, $\kappa_t(s,s',i)$ is the amount of overall interaction within the DGN in the $i^{th}$ dimension of inputs $x_s,x_{s'}\in \R^{d_{in}}$. Note that, thanks to the path-view, the `signal', i.e., $x_s(i)x_{s'}(i)$ gets separated from the `wire', i.e., the connections in the DGN, which is captured in the `$\kappa_t$' term. Further, the algebraic expression for $K_t$ in \eqref{eq:ktalg} applies to all DGNs (including the standard DNN with ReLU activations).

\textbf{Simplifying $\kappa$}, which contains the joint path activity given by $A(x_s,p_1)A(x_{s'},p_2)$, and the inter-path interaction given by $\ip{\varphi_{t,p_1}, \varphi_{t,p_2}}$, is the next step. Towards this end, we state and discuss Assumptions~\ref{assmp:mainone},~\ref{assmp:maintwo}.
\begin{assumption}\label{assmp:mainone}
$\Theta_0\stackrel{iid}\sim Ber\left(\frac{1}{2}\right)$ over the set $\{-\sigma,+\sigma\}$. 
\end{assumption}
\textbf{Decoupling of paths from one another}, which stands for the fact that the inner product $\ip{\varphi_{t,p_1}, \varphi_{t,p_2}}$ of two different paths $p_1,p_2$ is $0$ on expectation. This is captured in Lemma~\ref{lm:pathdot} below.
\begin{lemma}\label{lm:pathdot}
Under Assumption~\ref{assmp:mainone}, for paths $p,p_1,p_2\in \P, p_1\neq p_2$, at initialisation we have (i) $\E{\ip{\varphi_{0,p_1}, \varphi_{0,p_2}}}= 0$, (ii) ${\ip{\varphi_{0,p}, \varphi_{0,p}}}= d\sigma^{2(d-1)}$
\end{lemma}

\begin{assumption}\label{assmp:maintwo}
$\G_0$ is statistically independent of $\Theta_0$.
\end{assumption}
\textbf{Decoupling of gates and paths:} In a standard DNN with ReLU activations, the gating and path strengths are statistically dependent, because, conditioned on the fact that a given ReLU activation is \emph{on}, the incoming weights of that activation cannot be simultaneously all \emph{negative}. Assumption~\ref{assmp:maintwo} makes path strength statistically independent of path activity.

\begin{theorem}[DIP in DGN]\label{th:dgnexp}
Under Assumption~\ref{assmp:mainone}, ~\ref{assmp:maintwo}, and $\frac{4d}{w^2}<1$ it follows that
 \begin{align*}
&\E{K_0}=d\sigma^{2(d-1)}(x^\top x \odot \lambda_0)\\
&Var\left[K_0\right]\leq O\left(d^2_{in}\sigma^{4(d-1)}\max\{d^2w^{2(d-2)+1}, d^3w^{2(d-2)}\}\right)
\end{align*}
\end{theorem}

\textbf{Choice of $\sigma$:} Note that, in the case of gates taking values in $\{0,1\}$, $\lambda_0(s,s'), s,s'\in [n]$ is a measure of overlap of sub-networks that start at any given input node, end at the output node, and are active for both input examples $s,s'$.
Loosely speaking, say, in each layer $\mu\in(0,1)$ fraction of the gates are \emph{on}, then $\lambda_0(s,s)$ is about $(\mu w)^{d-1}$. Thus, to ensure that the signal does not blow up in a DGN, we need to ensure $(\sigma^2 \mu w)^{d-1}=1$, which means $\sigma=O\left(\sqrt{\frac1{\mu w}}\right)$. 

In the expression for $\E{K_0}$, note that the input Gram matrix $x^\top x \in \R^{n\times n}$ is separate from $\lambda_0\in \R^{n\times n}$,	 which is a Gram matrix of the active sub-networks. Thanks to the path-view, we do not lose track of the input information, which, is otherwise bound to happen if we were to choose a layer by layer view of DIP in DGNs.

\begin{figure*}
\resizebox{\textwidth}{!}{
\begin{tabular}{ccc}
\includegraphics[scale=0.4]{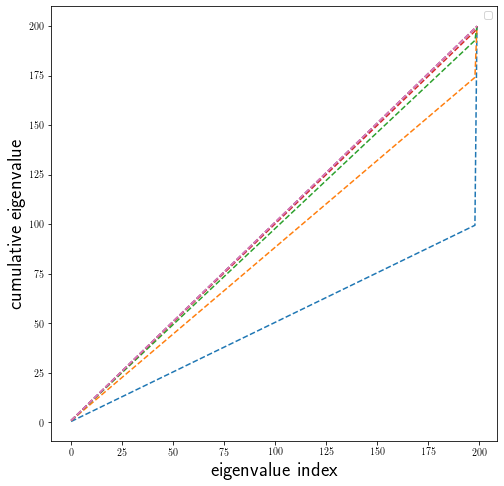}
&
\includegraphics[scale=0.4]{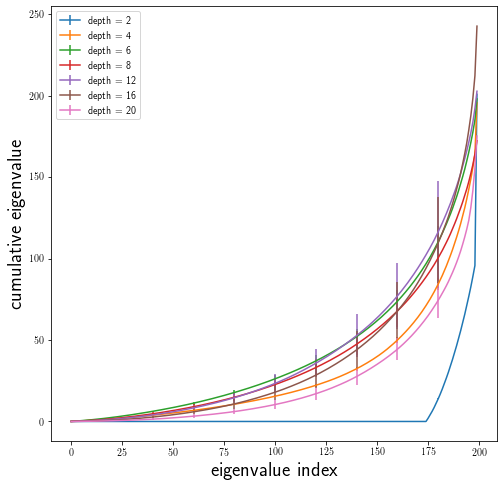}
&
\includegraphics[scale=0.4]{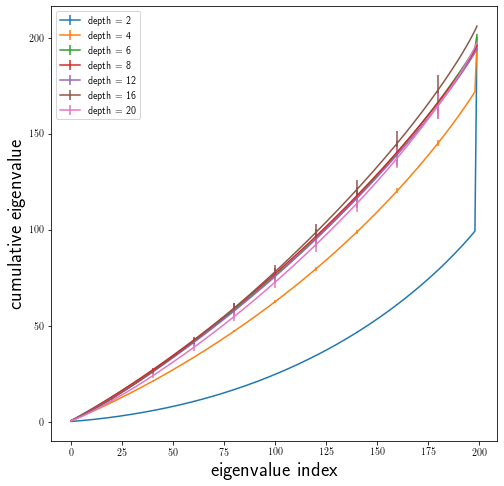}
\end{tabular}
}
\caption{Shows the plots for DGN-FRG with $\mu=\frac{1}{2}$ and $\sigma=\sqrt{\frac{2}{w}}$. The first plot in the left shows the ideal cumulative eigenvalue (e.c.d.f) for various depths $d=2,4,6,8,12,16,20$. Note that the ideal plot converges to identity matrix as $d$ increases. The second and third plot (from the left),  
plots respectively show the cumulative eigenvalues (e.c.d.f) for $w=500$ and $w=25$ respectively. Note that the e.c.d.f of higher width $w=500$ is better conditioned than the e.c.d.f of $w=25$.}
\label{fig:dgn-frg-gram-ecdf}
\end{figure*}

\textbf{Fixed Random Gating (DGN-FRG)} involves sampling the gates $\G_0$ from $Ber\left(\mu\right)$, and hence there is a random sub-network which is active for each input. Under FRG, we can obtain closed form expression for the `$\lambda(\cdot,\cdot,\cdot)$' term as below:

\begin{lemma}\label{lm:dgn-fra}  
 Under Assumption~\ref{assmp:mainone},~\ref{assmp:maintwo} and gates sampled iid $Ber(\mu)$, we have, $\forall s,s'\in[n]$

(i) $\mathbb{E}_p\left[\lambda_0(s,s)\right]=\bar{\lambda}_{self}=(\mu w)^{d-1}$

ii) $\mathbb{E}_p\left[\lambda_0(s,s')\right]=\bar{\lambda}_{cross}= (\mu^2w)^{d-1}$
\end{lemma}

\textbf{DIP in DGN-FRG:} For $\sigma=\sqrt{\frac{1}{\mu w}}$, we have:

\begin{align*}
\frac{\E{K_0}}{d}=\left[\begin{matrix}
\cdot&\cdot &\cdot &\cdot &\cdot \\ 
\cdot&\ip{x_s,x_s}\quad\quad &\cdot &\quad\quad\ip{x_s,x_{s'}}\mu^{d-1} &\cdot\\ 
\cdot&\ip{x_{s'},x_s}\mu^{d-1} &\cdot &\ip{x_{s'},x_{s'}} &\cdot \\
\cdot&\cdot &\cdot &\cdot &\cdot  
\end{matrix}\right]
\end{align*}

\textbf{Experiment $1$:} Consider the dataset $(x_s,y_s)_{s=1}^n\in \R\times \R$, where $x_s=1,\forall s\in [n]$, and $y_s\sim unif([-1,1])$, $n=200$. The input Gram matrix $x^\top x$ is a $n\times n$ matrix with all entries equal to $1$ and its rank is equal to 1. Since all the inputs are identical, this is the worst possible case for optimisation.

\textbf{Why increasing depth till a point helps ?} 
In the case of \textbf{Experiment $1$}, we have:

\begin{align}\label{eq:mat}
\frac{\E{K_0}}{d}=\left[\begin{matrix}
1 &\mu^{d-1} &\ldots &\mu^{d-1} &\ldots\\ 
\ldots &1 &\ldots &\mu^{d-1} &\ldots\\ 
\ldots &\mu^{d-1} &\ldots &1 &\ldots \\
\ldots &\mu^{d-1} &\ldots &\mu^{d-1} &1\\ 
\end{matrix}\right]
\end{align}

i.e., all the diagonal entries are $1$ and non-diagonal entries are $\mu^{d-1}$. Now, let $\rho_i\geq 0,i \in [n]$ be the eigenvalues of $\frac{\E{K_0}}{d}$, and let $\rho_{\max}$ and $\rho_{\min}$ be the largest and smallest eigenvalues. From the structure of \eqref{eq:mat}, one can easily show that $\rho_{\max}=1+(n-1)\mu^{d-1}$ and corresponds to the eigenvector with all entries as $1$, and $\rho_{\min}=(1-\mu^{d-1})$ repeats $(n-1)$ times, which corresponds to eigenvectors given by $[0, 0, \ldots, \underbrace{1, -1}_{\text{$i$ and $i+1$}}, 0,0,\ldots, 0]^\top \in \R^n$ for $i=1,\ldots,n-1$.

\textbf{Why increasing depth beyond a point hurts?} 
In \Cref{th:dgnexp}, note that for a fixed width $w$, as the depth increases, the variance of $K_0(s,s')$ increases, and hence the entries of $K_0$ deviates from its expected value $\E{K_0}$. Thus the structure of the Gram matrix degrades from \eqref{eq:mat}, leading to smaller eigenvalues.

\textbf{Numerical Evidence (Gram Matrix):} We fix arbitrary diagonal and non-diagonal entries, and look at their value averaged over $20$ run (see \Cref{fig:dgn-frg-gram-diag}). The actual values shown in bold indeed follow the ideal values shown in the dotted lines and the values are as per \eqref{eq:mat}).

\begin{figure}[h]
\resizebox{\columnwidth}{!}{
\begin{tabular}{cc}
\includegraphics[scale=0.4]{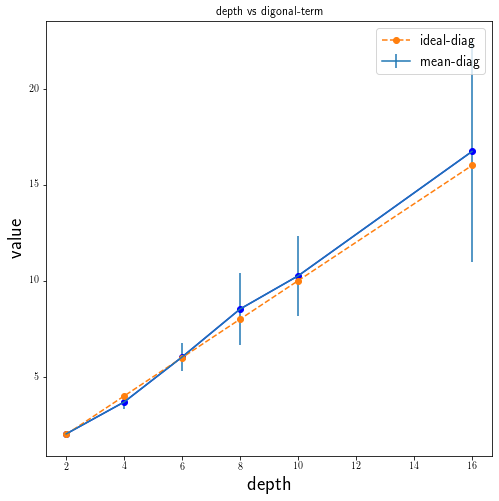}
&
\includegraphics[scale=0.4]{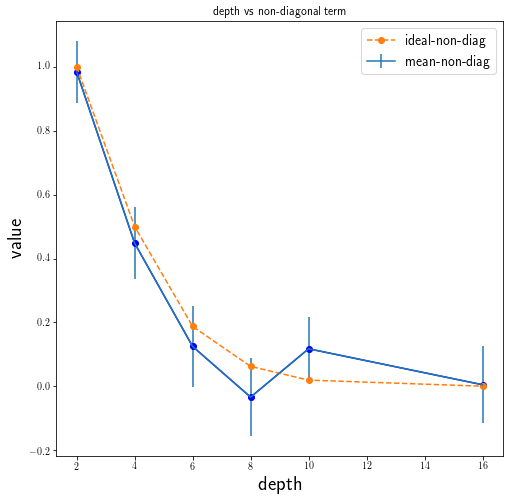}
\end{tabular}
}
\caption{For $w=500$, and $\mu=0.5$, (an arbitrary) diagonal (left plot) and non-diagonal (right plot) of the Gram matrix $K_0$ for dataset in Experiment~$1$ is shown. The plots are averaged over $20$ runs.}
\label{fig:dgn-frg-gram-diag}
\end{figure}

\textbf{Numerical Evidence (Spectrum):} 
Next, we look at the cumulative eigenvalue (e.c.d.f) obtained by first sorting the eigenvalues in ascending order then looking at their cumulative sum. The ideal behaviour (middle plot of \Cref{fig:dgn-frg-gram-ecdf}) as predicted from theory is that for indices $k\in[n-1]$, the e.c.d.f should increase at a linear rate, i.e., the cumulative sum of the first $k$ indices is equal to $k(1-\mu^{d-1})$, and the difference between the last two indices is $1+(n-1)\mu^{d-1}$. In \Cref{fig:dgn-frg-gram-ecdf}, we plot the e.c.d.f for various depths $d=2,4,6,8,12,16,20$ and two different width namely $w=25,500$. It can be seen that as $w$ increases, the difference between the ideal and actual e.c.d.f curves is less ($w=500$ when compared to $w=25$).

\textbf{Numerical Evidence (Role of Depth):} 
In order to compare how the rate of convergence varies with the depth in DGN-FRG network, we set the step-size $\alpha=\frac{0.1}{\rho_{\max}}$, $w=100$, and fit the data described in \textbf{Experiment $1$}. We use the vanilla SGD-optimiser. Note that if follows from \eqref{eq:basictraj} that the convergence rate is determined by a linear recursion, and choosing $\alpha=\frac{0.1}{\rho_{\max}}$ can be seen to be equivalent to having a constant step-size of $\alpha=0.1$ but dividing the Gram matrix by its maximum eigenvalue instead. Thus, after this rescaling, the maximum eigenvalue is $1$ uniformly across all the instances, and the convergence should be limited by the smaller eigenvalues. We also look at the convergence rate of the ratio $\frac{\norm{e_t}^2_2}{\norm{e_0}^2_2}$, and we observe that the convergence rate gets better with depth as predicted by theory (\Cref{fig:dgn-frg}).

\begin{figure*}
\resizebox{\textwidth}{!}{
\begin{tabular}{cccc}

\includegraphics[scale=0.4]{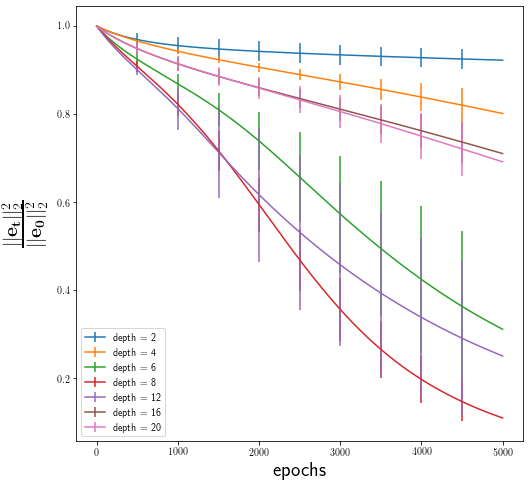}
&
\includegraphics[scale=0.4]{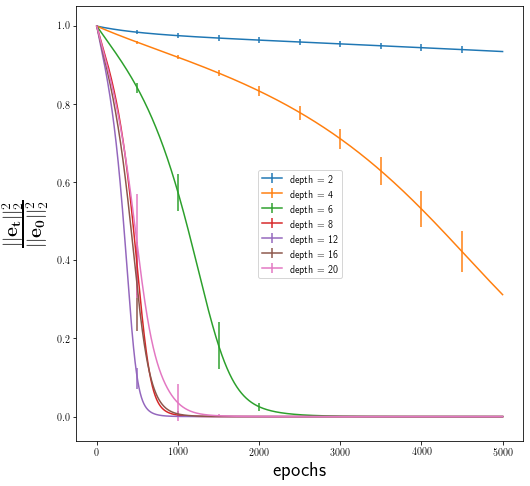}
&
\includegraphics[scale=0.4]{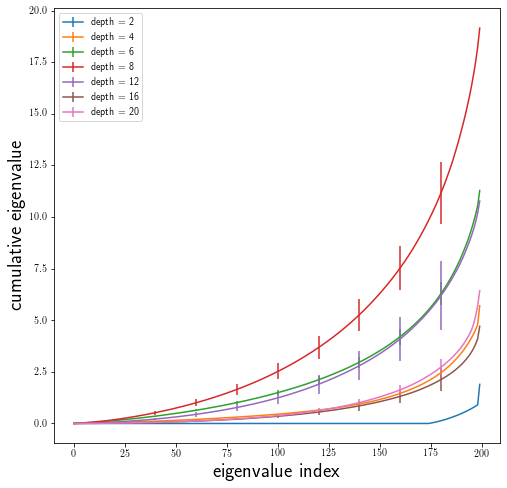}
&
\includegraphics[scale=0.4]{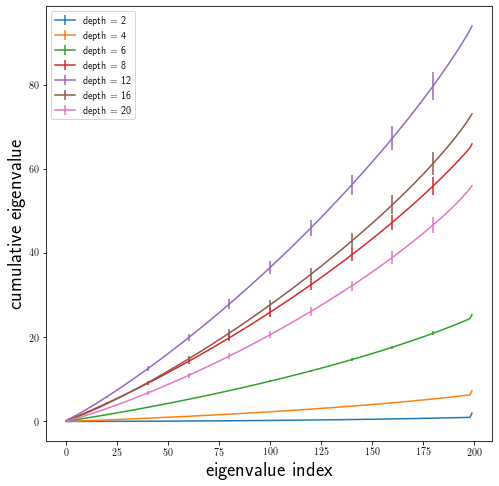}
\end{tabular}
}
\caption{Shows the plots for DGN-FRG. The left two plots show the convergence rates for $w=25$ and $w=500$. The right two values are showing the e.c.d.f obtained by first dividing the Gram matrix by their maximum eigenvalue. The plots are averaged over $5$ runs.}
\label{fig:dgn-frg}
\end{figure*}
\begin{figure*}
\resizebox{\textwidth}{!}{
\begin{tabular}{cccc}
\includegraphics[scale=0.4]{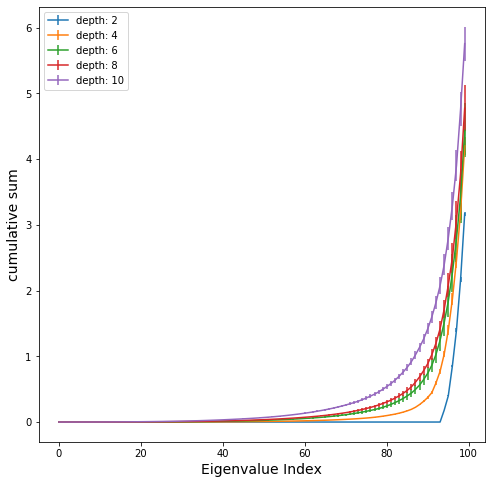}
&
\includegraphics[scale=0.4]{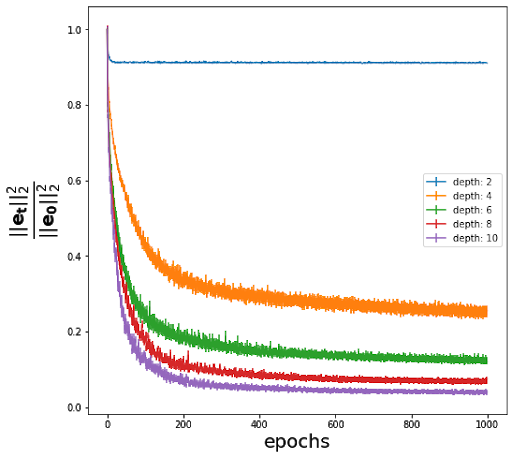}
&
\includegraphics[scale=0.4]{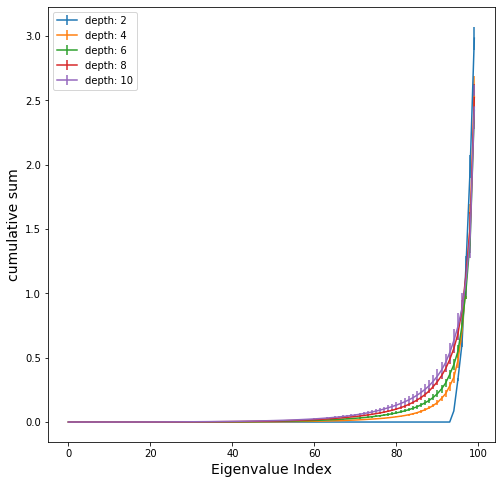}
&
\includegraphics[scale=0.4]{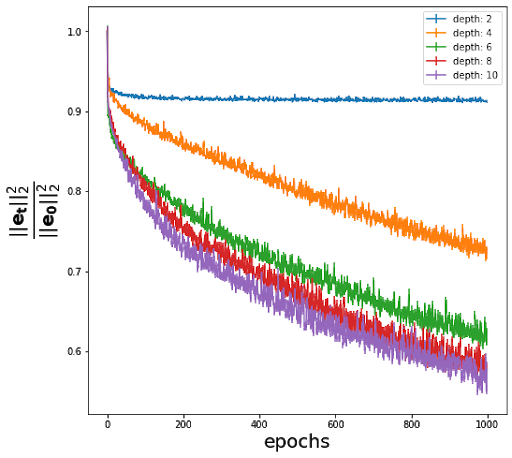}
\end{tabular}
}
\caption{The left two plots shows the e.c.d.f and convergence rates for various depth in GaLU networks $w=100$. The third and fourth plot from the left show the e.c.d.f and convergence rates for various depth in ReLU networks $w=100$. The plots are averaged over $5$ runs. }
\label{fig:galu-d}
\end{figure*}

\textbf{GaLU Networks:} Here, the gating values are obtained from a DNN with ReLU activation, parameterised by $\Tg\in \R^{d_{net}}$ (these weights are frozen). Now, let us define $\bar{\lambda}_{self}(s)\stackrel{def}=\mathbb{E}_{\Tg_0}\left[\lambda_0(s,s)\right]$, and $\bar{\lambda}_{cross}(s,s')\stackrel{def}= \mathbb{E}_{\Tg_0}\left[\lambda_0(s,s')\right]$. Note that, due to the inherent symmetry in weights (Assumption~\ref{assmp:mainone}) , we can expect roughly half the number of activations to be \emph{on}, and it follows that $\bar{\lambda}_{self}(s)\approx (\mu w)^{d-1}$ with $\mu\approx\frac12$. Also, let $\tau(s,s',l)\stackrel{def}=\sum_{i=1}^w G_{x_s,t}(l,i)G_{x_s',t}(l,i)$,  let $\eta\stackrel{def}=\max_s\left(\max_{s',l} \frac{\tau(s,s',l)}{\tau(s,s,l)}\right)$ be the maximum overlap between gates of a layer (maximum taken over over input pairs $s,s'\in[n]$ and layers $l\in [d]$), then it follows that $\max_{s,s'\in [n]} \frac{\bar{\lambda}_{cross}(s,s')}{\bar{\lambda}_{self}(s)}\leq \eta^{d-1}$. Thus,  we can see that while the non-diagonal entries of the $\E{K_0}$ decay at a different rates, the rate of decay is nonetheless upper bounded by $\eta^{d-1}$. Note that in DGN-FRG decay of non-diagonal terms is at a uniform rate given by $\mu^{d-1}$.

\textbf{Experiment $2$:} To characterise the optimisation performance of GaLU and ReLU networks, we consider the dataset $(x_s,y_s)_{s=1}^{n}\in \R^2\times \R$, where, $x_s\stackrel{iid}\sim unif(\left[-1,1\right]^2)$ and $y_s\stackrel{iid}\sim unif([-1,1])$, $n=100$. The results are shown in \Cref{fig:galu-d}. The rationale behind choosing this data set is that, we want the inputs to be highly correlated by choice.
 
 \textbf{GaLU Networks (Depth helps in training): }The trend is similar to DGN-FRG case, in that, both e.c.d.f as well as convergence get better with increasing depth. Here too we set the step-size $\alpha=\frac{0.1}{\rho_{\max}}$ (and use vanilla SGD). We also observe that in \textbf{Experiment $2$} \emph{GaLU networks optimise better than standard ReLU networks}, and it is also true that the e.c.d.f for the case of GaLU is better than that of ReLU. This can be attributed to the fact that, in ReLU network the dot product of two different active paths is not zero and hence the Gram matrix entries fall back to the algebraic expression for $K_t$ in \eqref{eq:ktalg}.

\section{Generalisation}\label{sec:generalisation}
\textbf{ReLU networks generalise better than GaLU:} We trained both ReLU and GaLU networks on standard MNIST dataset to close to $100\%$ accuracy. We observed that the GaLU network trains a bit faster than the ReLU network (see \Cref{fig:galu-relu} ). However, in test data we obtain accuracy of around $96.5\%$ and  $98.5\%$ for GaLU and ReLU networks respectively. A key difference between GaLU and ReLU networks is that, in ReLU networks the gates are adapting, i.e., $\G_t$ keeps changing with time. This leads us to the following natural question:
\begin{center}
\emph{Is gate adaptation key for generalisation performance?}
\end{center}
\begin{figure*}
\resizebox{\textwidth}{!}{
\begin{tabular}{cccc}
\includegraphics[scale=0.1]{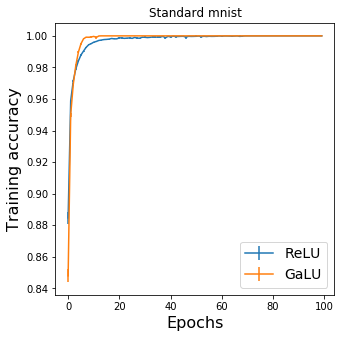}
&
\includegraphics[scale=0.1]{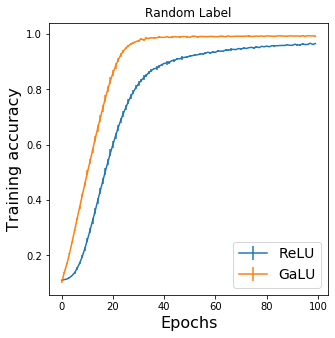}
&
\includegraphics[scale=0.1]{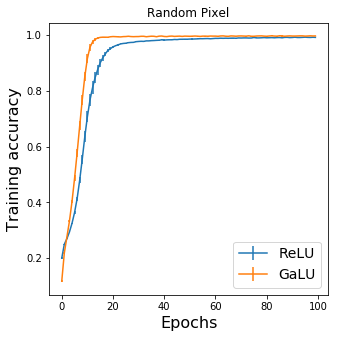}
&
\includegraphics[scale=0.1]{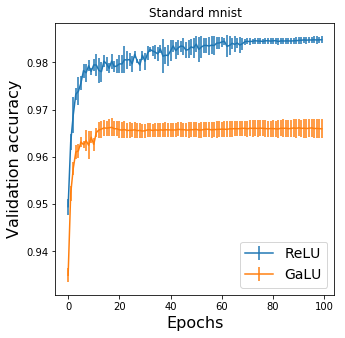}
	\end{tabular}
}
\caption{First three plots from the left show optimisation in ReLU and GaLU networks for standard MNIST, MNIST with random label and pixels  The right most plot shows generalisation of ReLU and GaLU networks in standard MNIST. Architecture we use is $d=7$, with layer widths from the first to last given by $512,512,256,256,128,64,10$ followed by a soft-max layer. In the case of GaLU there are two such network, wherein, one network is used to generate the gating values (whose weights are frozen) and the other network has the weights that are trained.}
\label{fig:galu-relu}
\end{figure*}

\textbf{Hidden features are in the sub-networks and are learned:} We consider ``Binary''-MNIST data set with two classes namely digits $4$ and $7$, with the labels taking values in $\{-1,+1\}$ and squared loss. We trained a standard DNN with ReLU activation ($w=100$, $d=5$). Recall from \Cref{sec:intro}, that $M_t=\Phi^\top_t\Phi_t$  (the Gram matrix of the features) and let $\widehat{M}_t=\frac{1}{trace(M_t)}M_t$ be its normalised counterpart. For a subset size, $n'=200$ ($100$ examples per class) we plot $\nu_t=y^\top (\widehat{M}_t)^{-1} y$, (where $y\in\{-1,1\}^{200}$ is the labeling function), and observe that $\nu_t$ reduces as training proceeds (see middle plot in \Cref{fig:path-norm}). Note that $\nu_t=\sum_{i=1}^{n'}(u_{i,t}^\top y)^2 (\hat{\rho}_{i,t})^{-1}$, where $u_{i,t}\in \R^{n'}$ are the orthonormal eigenvectors of $\widehat{M}_t$ and $\hat{\rho}_{i,t},i\in[n']$ are the corresponding eigenvalues. Since $\sum_{i=1}^{n'}\hat{\rho}_{i,t}=1$, the only way $\nu_t$ reduces is when more and more energy gets concentrated on $\hat{\rho}_{i,t}$s for which $(u_{i,t}^\top y)^2$s are also high. However, in $M_t=(x^\top x)\odot \lambda_t$, only $\lambda_t$ changes with time. Thus, $\lambda_t(s,s')$ which is a measure of overlap of sub-networks active for input examples $s,s'\in[n]$, changes in a manner to reduce $\nu_t$. We can thus infer that the \emph{right} active sub-networks are learned over the course of training.

\begin{figure}
\centering
\includegraphics[scale=0.5]{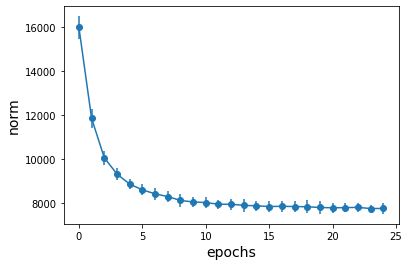}
\caption{Shows $\nu_t=y^\top (\widehat{M}_t)^{-1}y$, where $M_t=\Phi_t^\top \Phi_t$.}
\label{fig:path-norm}
\end{figure}

\textbf{DGNs with adaptable gates:} To investigate the role of gate adaptation,  we study parameterised DGNs of the general form given in \Cref{tb:dgn-parameterised}. The specific variants we study in the experiments are in \Cref{tb:dgn-family}.
\FloatBarrier
\begin{table}[h]
\resizebox{\columnwidth}{!}{
\centering
\begin{tabular}{|c|c|c|}\hline
Gating Network & Weight Network\\\hline
$z_{x,\Tg_t}(0)=x$ & $z_{x,\Tw_t}(0)=x$ \\\hline
 $q_{x,\Tg_t}(l)={\Tg_t(l)}^\top z_{x,\Tg_t}(l-1)$ & $q_{x,\Tw_t}(l)={\Tw_t(l)}^\top z_{x,\Tw_t}(l-1)$\\\hline
 $z_{x,\Tg_t}(l)=q_{x,\Tg_t}(l)\odot G_{x,\Tg}(l)$ & $z_{x,\Tw_t}(l)=q_{x,\Tw_t}(l)\odot G_{x,\Tg_t}(l)$ \\\hline
 \multicolumn{2}{|c|}{$\hat{y}_{t}(x)={\Tw_t(d)}^\top z_{x,\Tw_t}(d-1)$}\\\hline 
 \multicolumn{2}{|c|}{$\begin{aligned}\beta >0: G_{x,\Tg_t}(l,i)&=\chi_{\epsilon}(-\beta q_{x,\Tg_t}(l,i)), \\ \beta=\infty: G_{x,\Tg_t}(l,i)&=\mathbbm{1}_{\{q_{x,\Tg_t}(l,i)>0\}}\end{aligned}$}\\\hline 
\end{tabular}
}
\caption{A DGN with parameterised gates. Here, for $\epsilon\geq 0$, $\chi_{\epsilon}(v)=\frac{1+\epsilon}{1+\exp(v)}, \forall v\in \R$}
\label{tb:dgn-parameterised}
\end{table}

\FloatBarrier
\begin{table}[h]
\centering
\resizebox{\columnwidth}{!}{
\begin{tabular}{|c|c|c|}\hline
Terminology& Notation & Remarks\\\hline
ReLU & $\N(\Theta_t,\infty;\Theta_t)$ & $\Theta_t\in \R^{d_{net}}$\\\hline
GaLU (Frozen) &$\N(\Tg_{\dagger},\infty;\Tw_t)$ & $\Tw_t\in \R^{d_{net}}$, $\Tg_{\dagger}\in \R^{d_{net}}$\\\hline
Soft-ReLU &$\N(\Theta_t,\beta;\Theta_t)$ & $G\in(0,1)$ (not decoupled)\\\hline
Soft-GaLU &$\N(\Tg_t,\beta;\Tw_t)$ &  $G\in(0,1)$ (decoupled) \\\hline
\end{tabular}
}
\caption{Shows the variants of DGNs in the experiments. Here, $\dagger$ stands for frozen weights that are initialised by not trained.}
\label{tb:dgn-family}
\end{table}
\begin{figure*}
\resizebox{\textwidth}{!}{
\begin{tabular}{cccc}
\includegraphics[scale=0.1]{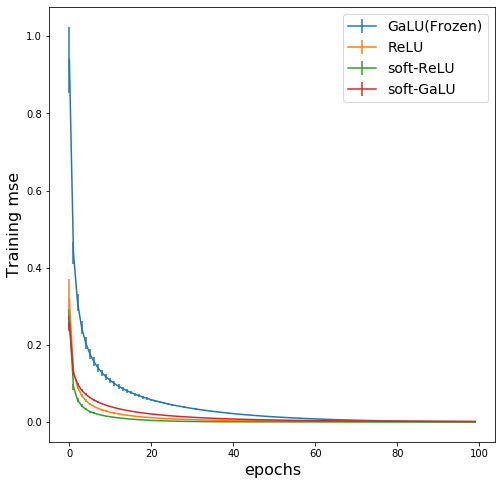}
&
\includegraphics[scale=0.1]{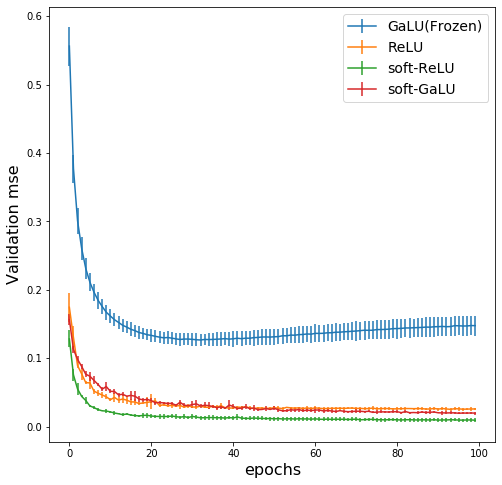}
&
\includegraphics[scale=0.1]{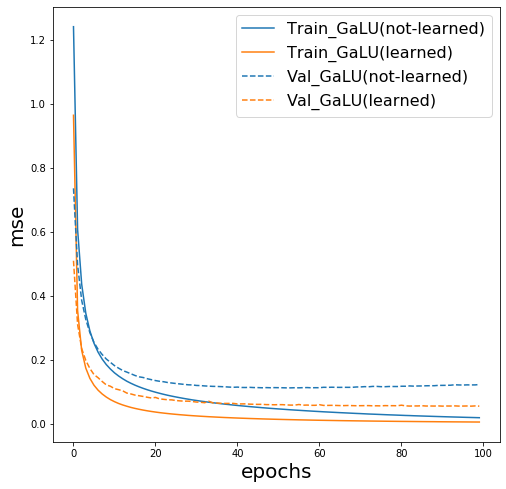}
&
\includegraphics[scale=0.1]{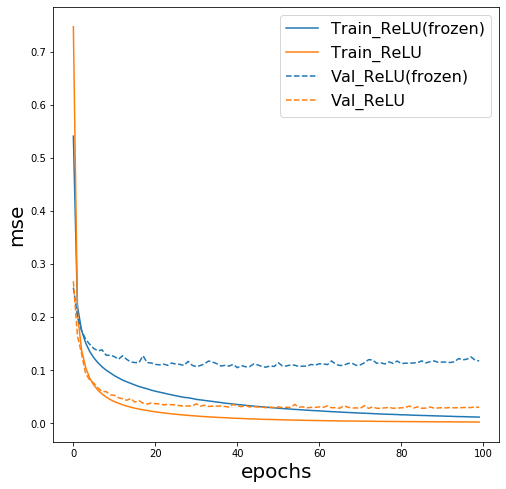}

\end{tabular}
}
\caption{The left two plots show respectively the training and generalisation in the $4$ different networks with $w=100$, $d=6$. Generalisation performance (dotted lines) of learned gates vs unlearned gates ($3^{rd}$ from left), adaptable gates vs frozen gates (rightmost). The plots are averaged over $5$ runs. }
\label{fig:adapt}
\end{figure*}

\begin{figure*}
\resizebox{\textwidth}{!}{
\begin{tabular}{ccccc}
\includegraphics[scale=0.1]{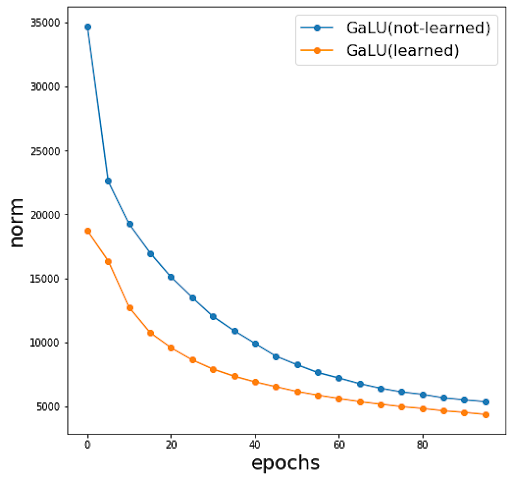}
&
\includegraphics[scale=0.1]{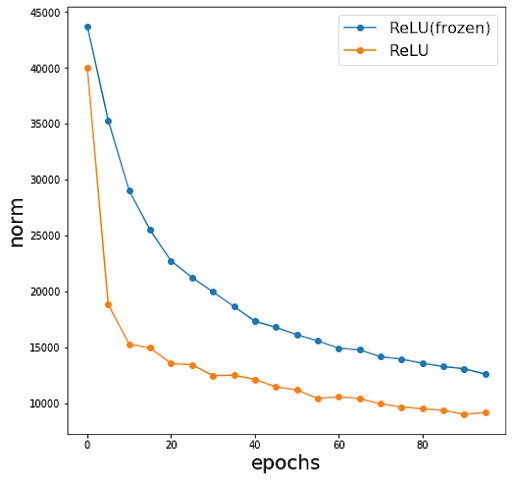}
&

\includegraphics[scale=0.18]{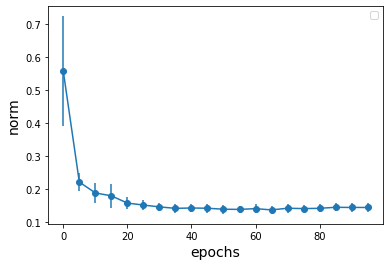}
&
\includegraphics[scale=0.18]{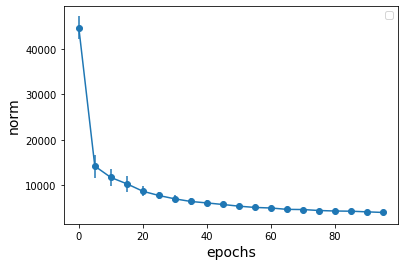}
\end{tabular}
}
\caption{Shows $\nu_t=y^\top (H_t)^{-1}y$ for (from the left) i) $H_t=\widehat{K}_t$, ii) $H_t=\widehat{K}_t$,  (iii) $H_t=K^a_t$ (iv) $H_t=\widehat{K^a_t}$. Here $K^a_t$ is the Gram matrix of activations in the soft-GaLU network.}
\label{fig:adapt-norm}
\end{figure*}

We now explain the idea behind the various gates (and some more) in \Cref{tb:dgn-family} as follows:

$1.$ The most general gate is called the \emph{soft-GaLU} gate denoted by $\N(\Tg,\beta;\Tw)$. Here, the gating values and hence the path activation levels are decided by $\Tg_t\in\R^{d_{net}}$, and the path strengths are decided by $\Tw_t\in \R^{d_{net}}$. This network has $2d_{net}$ parameters. 

$2.$ The standard DNN with ReLU gating is denoted by $\N(\Theta_t,\infty;\Theta_t)$, where $\infty$ signifies that the outputs are $0/1$ (see \Cref{tb:dgn-parameterised}). Here, both the gating (and hence path activation levels) and the path strengths are decided by the same parameter namely $\Theta_t\in\R^{d_{net}}$.

$3.$ $\N(\Theta_t,\beta;\Theta_t)$ is a DNN with what we call the \emph{soft-ReLU} gates, where the gating values are in $(0,1+\epsilon)$ instead of $0/1$. Here too, like the standard ReLU networks, both the gating values and the path strengths are decided by $\Theta_t\in\R^{d_{net}}$.

$4.$ $\N(\Tg_{\dagger}, \infty;\Tw_t)$ is what we call a GaLU-frozen DGN, where the gating parameters $\Tg\in \R^{d_{net}}$ are initialised but not trained. 

$5.$ $\N(\Tg_t,\beta;\Tw_{\dagger})$ is a network where only the gating parameters $\Tg_t\in \R^{d_{net}}$ are trainable and the parameters which dictate the path strengths namely $\Tw$ are initialised by not trained.

\textbf{Gradient of gate adaptation}

$\bullet$ \textbf{Fixing ReLU artefact:} We refer to the function $\chi_\epsilon(v)$ in \Cref{tb:dgn-parameterised} as the \emph{soft} gating function. Note that, the NTF matrix has two components given by $\Psi_t(m,s)={\phi_{x_s,\G_t }^\top\frac{\partial w_{t}} {\partial \theta(m)}+ \frac{\partial \phi_{x_s,\G_t }^\top}{\partial \theta(m)} w_{t}}$. In the case of ReLU activation, the gating values are either $0/1$, the activation levels are also $0/1$ and hence their derivative is $0$. In contrast, the `soft-gating' is differentiable, and hence if follows that $\frac{\partial \phi_{x_s,\G_t }^\top}{\partial \theta(m)} w_{t}\neq 0$, which is also accounted in the analysis.
 
$\bullet$ \textbf{For soft-GaLU:} Since there are two set of parameters (total $2d_{net}$) $\Psi_t^\top=[{\Psi^w}^\top_t,{\Psi^a}^\top_t]$ is a $n\times 2d_{net}$, we have $K_t=K^w_t+K^a_t$,  where $K^w_t={\Psi^w_t}^\top \Psi^w_t$, and $K^a_t={\Psi^a_t}^\top \Psi^a_t$.

$\bullet$ \textbf{For soft-ReLU:} $K_t={K^w_t}+{K^a_t}+{\Psi^w_t}^\top {\Psi^a_t}+{\Psi^a_t}^\top {\Psi^w_t}$.

\begin{definition} For a soft-GaLU DGN, using any $i\in[d_{in}]$, define $\delta(s,s')\stackrel{def}= \underset{{p\rsa i}}{\sum} \sum_{m=1}^{d_{net}}\frac{\partial A_{\Tg_0}(x_s,p)}{\partial \tg(m)} \frac{\partial A_{\Tg_0}(x_{s'},p)}{\partial \tg(m)}$.
\end{definition}

\begin{lemma} Under Assumptions~\ref{assmp:mainone},~\ref{assmp:maintwo}, in soft-GaLU networks we have: (i) $\E{K_0}=\E{K^w_0}+\E{K^a_0}$, 
 (ii) $\E{K^w_0}=\sigma^{2(d-1)} (x^\top x)\odot \lambda$,  (iii) $\E{K^a_0}=\sigma^{2d}  (x^\top x)\odot \delta$
\end{lemma}

\textbf{Adaptable gates generalise better:} In all experiments below, we use step-size $\alpha=1e^{-4}$, $w=100,d=6$ and we use \emph{RMSprop} to train.
\comment{
\FloatBarrier
\begin{figure}[h]
\resizebox{\columnwidth}{!}{
\begin{tabular}{c}
\includegraphics[scale=0.1]{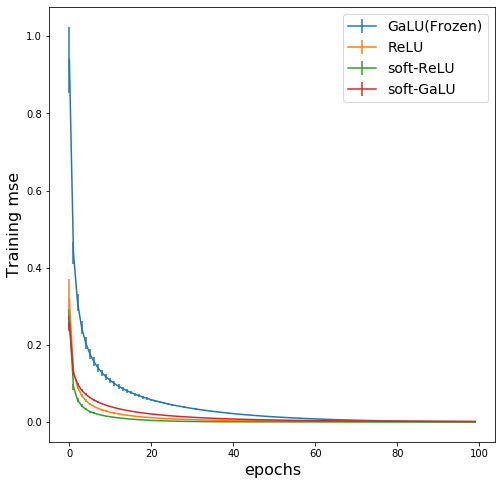}
\end{tabular}
}
\caption{Shows the training performance of the $4$ different networks in \textbf{Experiment $3$} with $w=100$, $d=6$.} 
\label{fig:adapt-4net-train}
\end{figure}
}

$\bullet$ \textbf{Experiment $3$:} On `Binary'-MNIST, we train four different networks ($w=100$, $d=6$), namely, $\N(\Tg_t,\beta=4;\Tw_t)$ (soft-GaLU), $\N(\Theta_t,\beta=4;\Theta_t)$ (soft-ReLU), $\N(\Theta_t,\infty;\Theta_t)$ (ReLU), and $\N(\Tg_{\dagger}, \infty;\Tw_t)$ (GaLU with frozen gates). We observe that the $3$ networks with adaptable gates generalise better the GaLU network with frozen gates as shown in \Cref{fig:adapt}.
\comment{
\FloatBarrier
\begin{figure}[h]
\resizebox{\columnwidth}{!}{
\begin{tabular}{c}
\includegraphics[scale=0.1]{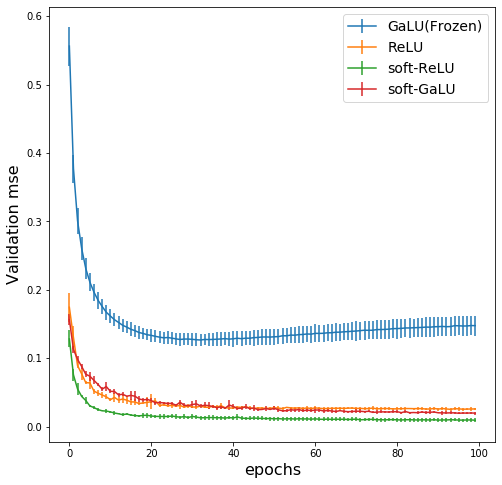}

\end{tabular}
}
\caption{Shows the generalisation performance of the $4$ different networks in \textbf{Experiment $3$} with $w=100$, $d=6$.} 
\label{fig:adapt-4net-gen}
\end{figure}
}

$\bullet$\textbf{Experiment $4$:} We train a frozen ReLU network $\N(\Tg_{\dagger},\infty;\Tw_t)$, where $\Tg_0=\Tw_0$ (weights chosen according to Assumption~\ref{assmp:maintwo}), and a standard ReLU network $\N(\Theta_t,\infty;\Theta_t )$ in which the gates adapt. We observe that generalisation is better when the gates adapt (see right most plot in \Cref{fig:adapt}).  
\comment{
\FloatBarrier
\begin{figure}[h]
\resizebox{\columnwidth}{!}{
\begin{tabular}{c}
\includegraphics[scale=0.1]{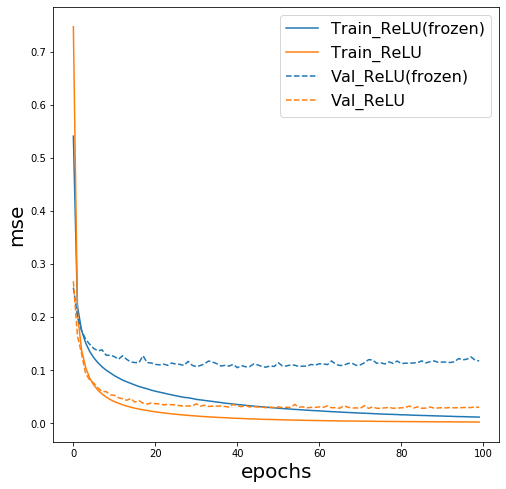}
\end{tabular}
}
\caption{ \textbf{Experiment $4$} with $w=100$, $d=6$. Generalisation performance (dotted lines) adaptable gates is better than frozen gates. The training performance is shown in bold lines. The plots are averaged over $5$ runs. } 
\label{fig:adapt-relu-frozen}
\end{figure}
}

$\bullet$ \textbf{Experiment $5$ (Lottery Ticket):} We train $\N(\Theta_t,\infty;\Theta_t)$ (a standard DNN with ReLU gates) till time $T=100$ epochs. We then use the weights to initialise $\Tg_0=\Theta_T$ for the network $\N(\Tg_{\dagger},\infty;\Tw_t)$ and train $\Tw_t$ (initialised independently) and we call this GaLU with learned gates. We also consider the case of GaLU network with non-learned gates, where we train $\N(\Tg_{\dagger},\infty;\Tw_t)$ where $\Tg_0$ and $\Tw_0$ is initialised according to Assumptions~\ref{assmp:mainone},~\ref{assmp:maintwo}. We observe that the learned gates generalise better than the non-learned case (see third plot from the left in \Cref{fig:adapt}). This shows that the real lottery is in the gates.
\comment{
\FloatBarrier
\begin{figure}[h]
\resizebox{\columnwidth}{!}{
\begin{tabular}{c}
\includegraphics[scale=0.1]{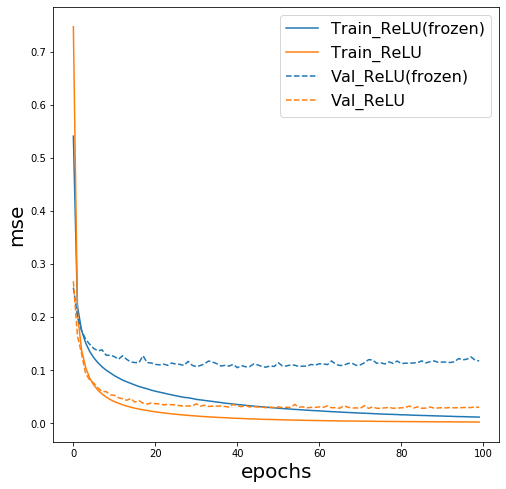}
\end{tabular}
}
\caption{ \textbf{Experiment $5$} with $w=100$, $d=6$. Generalisation performance (dotted lines) of learned gates is better than `non'-learned gates. The training performance is shown in bold lines. The plots are averaged over $5$ runs. } 
\label{fig:adapt-galu-relu}
\end{figure}
}

$\bullet$\textbf{Experiment $6$ (Lottery Ticket):} We consider $\N(\Tg_t,\infty;\Tw_{\dagger})$, where the weights corresponding to the strengths $\Tw$ are frozen, but the weights $\Tg$ that parameterise the gates are trained. We observed  a $56\%$ test performance in CIFAR-10 just by tuning the gates.  For this experiment, we used a \emph{convolutional} neural network, of the following architecture: input layer is $(32,32,3)$, followed by $conv(3, 3) : {64, 64, 128, 128}$, followed by $Flatten(), 256, 256, 10$.

$\bullet$ \textbf{Measure for generalisation:} In order to look for a possible explanation for the better generalisation performance (in the case of learned/adaptable gates), we plot $\nu_t=y^\top (H_t)^{-1}y$ (see \Cref{fig:adapt-norm}), for the following choices of $H_t$, namely  i) $H_t=\widehat{K}_t$,  $H_t=K^a_t$ and $H_t=\widehat{K^a_t}$,  Experiment $3$. We observe that, $\nu_t$ is smaller for the learned/adaptable gates when compared respectively to the case of non-learned/frozen gates. The behaviour of $\nu_t$ points to the fact that, when gates adapt, the underlying Gram matrices align their eigen-space in accordance with the labeling function.

\comment{
\begin{figure}
\resizebox{\columnwidth}{!}{
\begin{tabular}{c}
\includegraphics[scale=0.1]{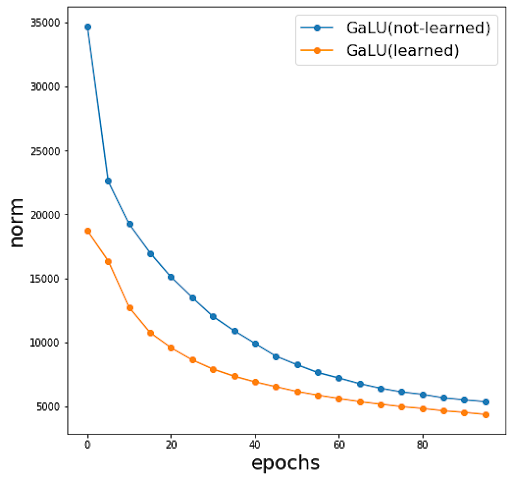}
\end{tabular}
}
\caption{Shows $\nu_t=y^\top (H_t)^{-1}y$ for $H_t=K_t$.}
\label{fig:adapt-norm-1}
\end{figure}

\begin{figure}
\resizebox{\columnwidth}{!}{
\begin{tabular}{c}
\includegraphics[scale=0.1]{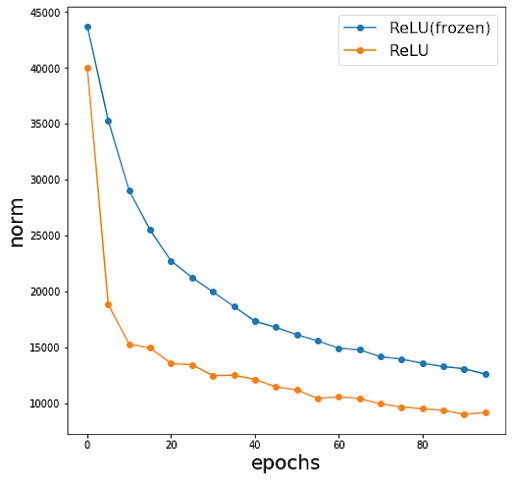}
\end{tabular}
}
\caption{Shows $\nu_t=y^\top (H_t)^{-1}y$ for $H_t=\widehat{K}_t$.}
\label{fig:adapt-norm-2}
\end{figure}

\begin{figure}
\resizebox{\columnwidth}{!}{
\begin{tabular}{c}
\includegraphics[scale=0.18]{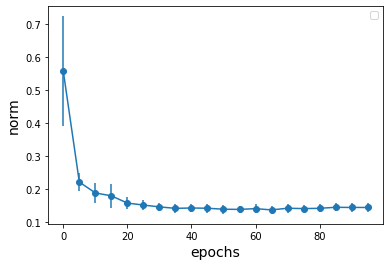}
\end{tabular}
}
\caption{Shows $\nu_t=y^\top (H_t)^{-1}y$ for $H_t=K^a_t$. Here, $K^a_t$ is the Gram matrix of activations in the soft-GaLU network.}
\label{fig:adapt-norm-3}
\end{figure}

\begin{figure}
\resizebox{\columnwidth}{!}{
\begin{tabular}{c}
\includegraphics[scale=0.18]{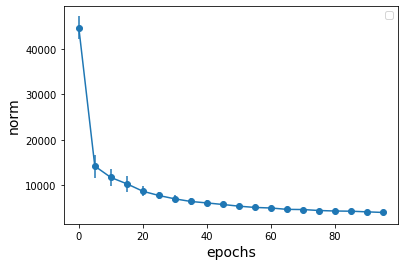}
\end{tabular}
}
\caption{Shows $\nu_t=y^\top (H_t)^{-1}y$ for $H_t=\widehat{K^a_t}$. Here, $K^a_t$ is the Gram matrix of activations in the soft-GaLU network.}
\label{fig:adapt-norm-4}
\end{figure}
}

$\bullet$ \textbf{Experiment $7$:} We train convolution networks with soft-ReLU gates (i.e., $\chi_{\epsilon}(v)$) for various values of $\beta$ and $\epsilon$ on CIFAR-10 dataset (the architecture is same as the one described in \textbf{Experiment $6$}).  For moderately high values of $\beta$ we obtain generalisation performance of $72\%$ which is comparable to what we obtain for DNNs (identical architecture, and hyper-parameters) with standard ReLU activation.
\FloatBarrier
\begin{table}[h]
\centering
\resizebox{\columnwidth}{!}{
\begin{tabular}{|c|c|c|}\hline
Network
&Training Accuracy
&Test Accuracy\\\hline
Deep Linear Network
&0.5042
&0.3939\\\hline
Relu
&0.99
&0.71\\\hline

$\beta=$1, $\epsilon=$0.1
&0.99
&0.59\\\hline
$\beta=$2, $\epsilon=$0.1
&0.99
&0.64\\\hline
$\beta=$4, $\epsilon=$0.1
&0.99
&0.68\\\hline
$\beta=$8, $\epsilon=$0.1
&0.99
&0.71\\\hline
$\beta=$12, $\epsilon=$0.1
&0.99
&0.71\\\hline
$\beta=$16, $\epsilon=$0.1
&0.99
&0.72\\\hline
$\beta=$20, $\epsilon=$0.1
&0.99
&0.72\\\hline
$\beta=$24, $\epsilon=$0.1
&0.99
&0.72\\\hline
$\beta=$1, $\epsilon=$0.4
&0.99
&0.56\\\hline
$\beta=$2, $\epsilon=$0.4
&0.99
&0.63\\\hline
$\beta=$4, $\epsilon=$0.4
&0.99
&0.68\\\hline
$\beta=$8, $\epsilon=$0.4
&0.99
&0.71\\\hline
$\beta=$12, $\epsilon=$0.4
&0.99
&0.71\\\hline
$\beta=$16, $\epsilon=$0.4
&0.99
&0.71\\\hline
$\beta=$20, $\epsilon=$0.4
&0.99
&0.71\\\hline
$\beta=$24, $\epsilon=$0.4
&0.99
&0.71\\\hline
\end{tabular}
}
\caption{Shows performance of various networks for CIFAR 10 dataset. The results are averaged over $5$ runs. The reported results are mean of the best performance obtained in each run. The optimiser used was \emph{Adam} with step-size equal to $3e^{-4}$.}
\label{tb:cifar}
\end{table}

\subsection{Preliminary analysis of gate adaptation:} 
Recall that $\G_t\stackrel{def}=\{G_{x_{s},t}(l,i) \in [0,1], \forall s\in[n],l\in[d-1],i\in[w]\}$. We now define the following:

$\bullet$ \textbf{Active Gates:} For an input $x_{s}\in \R^{d_{in}}$, and a threshold value $\tau_{\A}\in (0,1+\epsilon)$, define $\G_t^{\A}(x_s,\tau_{\A})\stackrel{def}=\left\{G_{x_s,t}(l,i)\colon G_{x_s,t}(l,i)> \tau_{\A}, l\in[d-1],i\in[w]\right\}$. These are the gates that are \emph{on} (i.e., more than threshold $\tau_{\A}$) for input $x_s\in\R^{d_{in}}$.

$\bullet$ \textbf{Sensitive Gates:} For an input $x_{s}\in \R^{d_{in}}$, and a threshold value $\tau_{\S}>0$, define $\G_{t}^{\S}(x_s,\tau_{\S})\stackrel{def}=\cup_{m=1}^{d_{net}}\left\{G_{x_s,t}(l,i)\colon \left|\frac{\partial G_{x_s,t}(l,i)}{\partial \tg(m)}\right| >\tau_{\S},l\in[d-1],i\in[w] \right\}$. These are set of gates that are sensitive to changes in anyone of the $\tg(m),m\in[d_{net}]$ tunable parameters that control the gates.

$\bullet$ \textbf{Relation between sensitive and active gates:} From the nature of the soft-gating function $\chi_{\epsilon}(v)$ it follows that for any given $\tau_{\A}\in(0,1+\epsilon)$, it follows that $\G_t^{\A}(x_s,\tau_{\A})\cap \G_t^{\S}(x_s,\tau_{\S})=\emptyset,\forall \tau_{\S}>\frac{d \chi_{\epsilon}(v)}{d v}|_{v=\chi^{-1}_{\epsilon}(\tau_{\A})}$. Also, note that as $\tau_{\A}\ra (1+\epsilon)$, $\tau_{\S}\ra 0$.

$\bullet$ \textbf{Sensitivity of Activations:} For a path $p$, and a gating parameter $\tg(m),m\in[d_{net}]$ we have $\frac{\partial A_{\Tg_t}(x_s,p)}{\partial \tg(m)}$ to be equal to
\begin{align}\label{eq:sensitivity}
\sum_{l=1}^{d-1} \Big(\frac{\partial G_{x_s,\Tg_t}(l,p(l))}{\partial \tg(m)} \Big)\Big(\Pi_{l'\neq l} G_{x_s,\Tg_t}(l',p(l'))\Big)
\end{align}

In what follows we assume that $\tau_{\S}>\frac{d \chi_{\epsilon}(v)}{d v}|_{v=\chi^{-1}_{\epsilon}(\tau_{\A})}$.

$\bullet$ \textbf{Active Sub-Network:} Which paths are active for input $x_s\in\R^{d_{in}}$?\quad Choose a threshold $\tau_{\A}$ close to $1$. The paths that pass through gates in $\G_t^{\A}(x_s,\tau_{\A})$ do not matter much in gate adaptation because they are already \emph{on}, and are responsible for holding the memory for input $x_s\in\R^{d_{in}}$. In particular, \eqref{eq:sensitivity} evaluates close to $0$ for such paths because $\left|\frac{\partial G_{x_s,\Tg_t}(l,p(l))}{\partial \tg(m)}\right|<\frac{d \chi_{\epsilon}(v)}{d v}|_{v=\chi^{-1}_{\epsilon}(\tau_{\A})}$.

$\bullet$ \textbf{Sensitive Sub-Network:} Which paths are learning for input $x_s\in\R^{d_{in}}$? \quad Those paths that have one gate from $\G^{\S}_t(x_s,\tau_{\S})$ and the rest of the $(d-2)$ gates from the set  $\G^{\A}_t(x_s,\tau_{\A})$. For such paths, the magnitude of at least one of the $(d-1)$ terms in \eqref{eq:sensitivity} will be greater than $\tau_{\S}(\tau_{A})^{(d-2)}$, and the rest of the $(d-2)$ terms will contain a term whose magnitude is less than $\frac{d \chi_{\epsilon}(v)}{d v}|_{v=\chi^{-1}_{\epsilon}(\tau_{\A})}$ component and hence will contribute less to the summation.

$\bullet$ \textbf{Role of $\beta$} for now is limited to an analytical convenience, especially to address the non-differentiability artefact in ReLU gates. The ideas is that as $\beta\ra\infty$, the analysis applies more sharply to networks with ReLU gates.

\section{Understanding the role of convolutions and pooling operations}\label{sec:conv}
In this section, we will use the frameworks of DGNs and ``path-view'' to obtain insights about (i) convolutional layers and (ii) pooling: global average pooling\footnote{The arguments can be extended to $\max$-pooling with technical modifications.}. In this section, we continue to be in the DGN setup, i.e., we will have separate parameterisations $\Tg$ and $\Tw$, and assume that Assumptions~\ref{assmp:mainone}, \ref{assmp:maintwo} hold. However, we impose additional restrictions to account for the presence of convolutional and pooling layers, which, we describe below.

\textbf{Circular Convolutional layers:}

$1.$ We assume that, the initial $0<L<d$ layers are convolutional layers. In particular, each layer uses a $1$-dimensional kernel of width $0<\hat{w}<d_{in}$, and the output of each layer is a $d_{in}$-dimensional vector.

$2.$ We consider circular convolutional operations instead of zero padding, i.e., during the convolution operation, say index $i$ exceeds $d_{in}$ then it will be considered as $i-d_{in}$, and in the case when a negative index is required, i.e., if index $i<0$ is needed, then $d_{in}+i$ will be used instead. We illustrate this circular convolution with the help of \Cref{fig:circconv}, wherein, $\hat{w}=2$, $d_{in}=3$. Here, $\theta^{(l)}(i),l=1,\ldots,L-1, i=1,2$ are the weights, and the final layer weight $\theta^{(L)}=\left[\frac{1}{d_{in}},\ldots, \frac{1}{d_{in}}\right]^\top\in \R^{d_{in}}$ in the case of global average pooling. Note that, in \Cref{fig:circconv}, we have used only one network, and we have also used a simpler and different notation for the weights: this is because, in DGN (with circular convolutions), both the gating network parameterised by $\Tg$ and the weights network parameterised by $\Tw$ will have identical architecture, and in order to explain just the circular convolution alone more clearly, we have used a simpler notation for the weights and have left the gating information unspecified in \Cref{fig:circconv}.

\begin{figure}
\centering
\includegraphics[scale=0.25]{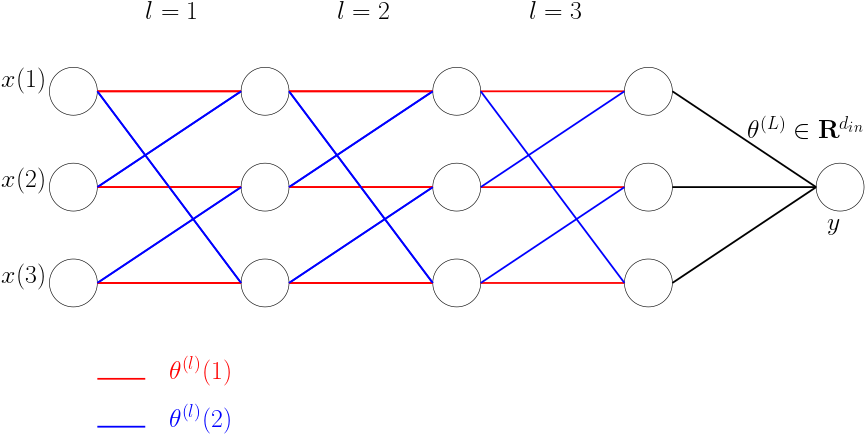}
\caption{Shows a circular convolutional network with $d_{in}=3$ and kernel size $\hat{w}=2$. Note that there are only $8$ unique path strengths in this example (in the case of global average pooling).}
\label{fig:circconv}
\end{figure}

\textbf{Path Sharing:} With the understanding of circular convolution in the background, we now investigate the similarity of two inputs $x_s\in \R^{d_{in}}$ and $x_{s'}\in \R^{d_{in}}$ after they pass through $L$ convolutional layers. To be specific, let $x_s(L)\in \R^{d_in}$ and $x_{s'}(L)\in \R^{d_{in}}$ be the outputs obtained after the $L$ convolutional layers. Note that $x_s(L)=(x_s(L,i),i\in[d_{in}])\in \R^w$ is a $d_{in}$-dimensional vector with $i=1,\ldots,d_{in}$ components, wherein, the $i^{th}$ component is obtained by circular convolution using a kernel of size $0<\hat{w}<d_{in}$. Further, we restrict our attention to the first $L$ layers which perform the convolution operations. We are interested in investigating the following: 
\begin{align}
\E{\ip{x_s(L),x_{s'}(L) }}=\sum_{i=1}^{d_{in}} \E{x_s(L,i)x_{s'}(L,i)}
\end{align}
Given the randomised and symmetric nature of the weight initialisation, without loss of generality, it is sufficient to study $\E{x_s(L,1)x_{s'}(L,1)}$, i.e., it is enough to consider the case of $L$ convolutions with kernel of size $\hat{w}$ followed by a global average pooling. We now make the following observations:

$1.$ There are $p=1,\ldots,\hat{P}=d_{in}\hat{w}^{d-1}$ paths.

$2.$ There are $k=1,\ldots,\hat{B}=\hat{w}^{d-1}$ unique path strengths. This is due to the fact that the same path strength repeats $d_{in}$ times. For instance, in \Cref{fig:circconv}, the path strength $\theta^{1}(1)\theta^{1}(2)\theta^{1}(3)\frac{1}{d_{in}}$ repeats $3$ times.

$3.$ Paths can be grouped into bundles $b_k,k\in[\hat{w}^{d-1}]$, wherein, bundle $b_k$ comprises of $d_{in}$ paths, all of which have the same path strength. Without loss of generality, $b_k$ comprises of paths $(k-1)d_{in}+1,\ldots, kd_{in}$.

$4.$ The path strength $w_t=(w_t(b_1),\ldots, w_t(b_{\hat{B}}))\in \R^{\hat{P}}$, where $w_t(b_k)=(w_t(p),p=(k-1)d_{in}+1,\ldots,kd_{in})\in \R^{d_{in}}$. 

$5.$ The output $x_s(L,1)=\phi^\top_{x_s,\G_t} w_t$.

\begin{lemma}\label{lm:invariance}
At $t=0$, under Assumptions~\ref{assmp:mainone},\ref{assmp:maintwo}, convolutional layers with global average pooling at the end causes translational invariance.
\begin{align*}
&\E{x_s(L,1)x_{s'}(L,1)}\\&=\frac{\sigma^{2(d-1)}}{d^2_{in}}\sum_{k=1}^{\hat{B}} \sum_{p_1,p_2\in b_k}  \Big( x(p_1(0),s) A(x_s,p_1)\\
&\quad\quad \quad\quad \quad\quad x(p_2(0),s') A(x_{s'},p_2) \Big)
\end{align*}
\end{lemma}

\textbf{Remark:} Now, for $i\in\{0,\ldots, d_{in}-1\}$, let $z^{(i)}\in \R^{d_{in}}$ be the clockwise rotation of $z\in \R^{d_{in}}$ by $i$ co-ordinates, and let $x^{(i)}\in \R^{d_{in}\times n}$ be the data matrix obtained by clockwise rotation of the columns of the data matrix $x\in \R^{d_{in}\times n}$ by $i$ co-ordinates. Then, we have

\begin{align*}
&\E{x_s(L,1)x_{s'}(L,1)}\\
&=\frac{\sigma^{2(d-1)}}{d^2_{in}}\sum_{k=1}^{\hat{B}} \sum_{i=1}^{d_{in}} \sum_{p\in b_k}   \Big(x(p(0),s) A(x,p) \\ 
&\quad\quad \quad\quad \quad\quad x^{(i)}(p(0),s') A(x^{(i)}_{s'},p) \Big)
\end{align*}
The term $\sum_{p\in b_k}  x(p(0),s) A(x,p) x^{(i)}(p(0),s') A(x^{(i)}_{s'},p)$ is translation invariant.

\comment{
\textbf{Claim $2$:} At $t=0$, under Assumptions~\ref{assmp:mainone},\ref{assmp:maintwo}, convolutional layers with $\max$-pooling at the end causes translational invariance.

\begin{proof}
The proof follows in a manner similar to \textbf{Claim $1$} made for the case of global average pooling. However, the technical challenge is the following: in the case of $\max$-pooling, only one of the $d_{in}$ paths connecting the $(L-1)^{th}$ layer to the output node is \emph{on}. This path connects the ``$\max$" node in the $(L-1)^{th}$ layer to the output node. This can be accounted in the calculations by setting the path strength to be $0$ for those paths that do not pass through the ``$\max$" node in the $(L-1)^{th}$ layer.  We make the following observations about $M$:

$1.$ For each bundle $b_k, k\in[\hat{B}]$, $\exists$ unique indices $i(k), j(k)\in [d_{in}]$ such that $M((k-1)d_{in}+i(k), (k-1)d_{in}+j(k))=\sigma^{2(d-1)}$, and rest of the entries of $M$ are $0$.

$2.$ $M(p_1,p_2)=\frac{\sigma^{2(d-1)}}{d^2_{in}}$, if $p_1$ and $p_2$ belong to the same bundle continues to hold trivially due to observation $1$ (of the current claim).

And by going through reductions similar to \textbf{Claim $1$}, we have

\begin{align*}
&\E{x_s(L,1)x_{s'}(L,1)}\\&=\phi^\top_{x_s,\G_0} M \phi^\top_{x_{s'},\G_0}\\
&=\sum_{p_1,p_2=1}^{\hat{P}} \Big(x(p_1(0),s) A(x_s,p_1) \\
&\quad\quad \quad\quad \quad\quad x(p_2(0),s') A(x_{s'},p_2) M(p_1,p_2)\Big)\\
&=\frac{\sigma^{2(d-1)}}\sum_{k=1}^{\hat{B}} \sum_{p_1,p_2\in b_k}  \Big( x(p_1(0),s) A(x_s,p_1)\\
&\quad\quad \quad\quad \quad\quad x(p_2(0),s') A(x_{s'},p_2) \Big)
\end{align*}
Now, for $i\in\{0,\ldots, d_{in}-1\}$, let $z^{(i)}\in \R^{d_{in}}$ be the clockwise rotation of $z\in \R^{d_{in}}$ by $i$ co-ordinates, and let $x^{(i)}\in \R^{d_{in}\times n}$ be the data matrix obtained by clockwise rotation of the columns of the data matrix $x\in \R^{d_{in}\times n}$ by $i$ co-ordinates. Then, we have

\begin{align*}
&\E{x_s(L,1)x_{s'}(L,1)}\\
&=\frac{\sigma^{2(d-1)}}{d^2_{in}}\sum_{k=1}^{\hat{B}} \sum_{i=1}^{d_{in}} \sum_{p\in b_k}   \Big(x(p(0),s) A(x,p) \\ 
&\quad\quad \quad\quad \quad\quad x^{(i)}(p(0),s') A(x^{(i)}_{s'},p) \Big)
\end{align*}
The term $\sum_{p\in b_k}  x(p(0),s) A(x,p) x^{(i)}(p(0),s') A(x^{(i)}_{s'},p)$ is translation invariant.

\end{proof}
}
\section{Related Work}
\cite{dudnn} show that in fully connected DNNs with $w=\Omega(poly(n)2^{O(d)})$, and in residual neural networks (ResNets) with $w=\Omega(poly(n,d))$ gradient descent converges to zero training loss. \cite{dudnn} claim to demystify the second part of what we called the depth phenomena (``why deeper networks are harder to train"), since, the dependence on the number of layers improves exponentially for ResNets. Our optimisation results are weaker than \cite{dudnn} in the sense that we consider only DGNs with decoupling assumptions. However, we show both parts of the depth phenomena, in particular why increasing depth till a point helps training. 

\comment{Further, the \emph{algebraic} nicety due to \emph{Hadamard} product decomposition of the Gram matrix is a useful take away. In addition, the connection to how the sub-network overlap is a conceptual gain. 
In comparison to \cite{dudln} the gain in our work is that, thanks to the path-view, we obtain a single expression for $\E{K_0}$ which can be applied to deep linear networks, GaLU networks and any networks whose gating values are known and fixed. Both \cite{dnn,dln} are analytically more involved, in that they provide guaranteed rates of converges with high probability, and in comparison, our work has stopped with the variance calculation.
}
In comparison to \cite{sss} who were the first to initiate the study on GaLU networks, we believe, our work has made significant progress. We introduced adaptable gates, and showed via experiments, that, gate adaptation is key in learning, thereby showing a clear separation between GaLU and ReLU networks. To support the claim, we have used idea from \cite{arora}, in that, we measure $\nu_t=y^\top {K_t}^{-1}y$ to show that the eigen spaces indeed align with respect to the labelling function.

In comparison to \cite{lottery}, we also show in our experiments that the winning lottery is in the gating pattern, which, in the case of ReLU networks is inseparable from the weights. However, our experiments show that the weights can be reinitialised if we have the learned gating pattern.

\section{Conclusion }
\comment{
In this paper, we looked at the gradient descent (GD) procedure to minimise the squared loss in deep neural networks. Prior literature \cite{dudnn} makes trajectory analysis (wherein the dynamics of the error terms are studied) to show that GD achieves zero training error. In this paper, we introduced to important conceptual novelties namely deep gated networks (DGNs) and path-view, to obtain additional insights about GD in the context of trajectory analysis. In particular, our theory threw light on i) the depth phenomena and ii) gate adaptation, i.e., the role played by the dynamics of the gates in generalisation performance.

The path-view lead following gains: (i) an explicit expression of information propagation in DGNs where in the input signal and the wires, i.e., the deep network itself are separated. This is unlike the conventional layer by layer approach, wherein, the input is lost in the hidden layers, (ii) explicitly identifying the role of sub-networks in training and generalisation of deep networks, so much so that, we can go so far as to say that the actual \emph{hidden features are in the paths and the sub-networks and not just the final layer output}, (iii) explicit identification of twin gradient flow, wherein, one component of the gradient flow to train the paths keeping the sub-network constant and the other component of the gradient takes care of learning the gating values.

We looked at various DGNs with adaptable gates and we observed  in experiments that the adaptable/learned gates generalise better than non-adapting/non-learned gates.  Based on our theory and experiments, we conclude that \emph{understanding generalisation would involve a study of gate adaptation}.
}
\comment{
In this paper, we introduced to important conceptual novelties namely deep gated networks (DGNs) and path-view, to obtain additional insights about gradient descent in deep learning. The path-view lead to the following gains: (i) an explicit expression of information propagation in DGNs (ii) explicitly identifying the role of sub-networks in training and generalisation of deep networks, (iii) explicit identification of twin gradient flow, wherein, one component of the gradient flow to train the path strengths keeping the sub-network constant and the other component of the gradient takes care of learning the gating values. Using the path-view and the DGNs, we showed  i) the depth helps is equivalent to whitening of data and increasing depth beyond degrades the spectrum of the Gram matrix at initialisation, and ii) gate adaptation, i.e., the role played by the dynamics of the gates is important for generalisation performance.

We looked at various DGNs with adaptable gates and we observed  in experiments that the adaptable/learned gates generalise better than non-adapting/non-learned gates.  Based on our theory and experiments, we conclude that \emph{understanding generalisation would involve a study of gate adaptation}.
}

In this paper, we introduced two important conceptual novelties namely deep gated networks (DGNs) and ``path-view", to obtain additional insights about gradient descent in deep learning. Using these two novel concepts, we achieved the following:

 (i) resolution to the depth phenomena for DGNs under decoupling assumption. In particular, our results showed that increasing depth is equivalent to whitening of data and increasing depth beyond a point degrades the spectrum of the Gram matrix at initialisation.
 
 (ii) each input example has a corresponding active sub-network, which are learned when the gates adapt.
 
 (iii) a preliminary theory to analyse gate adaptation. Our analysis points out to the presence of two complementary networks for each input example, one being the active sub-network which holds the memory for that input example and the other being the sensitivity sub-network of gates that are adapting.
 
(iv) we looked at various DGNs with adaptable gates and we observed  in experiments that the adaptable/learned gates generalise better than non-adapting/non-learned gates.  

Based on our theory and experiments, we conclude that :

(a) \emph{Hidden features are in the active sub-networks,} which are in turn decided by the gates.

(b) \emph{Understanding generalisation would involve a study of gate adaptation.}

\comment{
Let $\gamma>0$ be a threshold value, and let $G_{x_s,\Tg_t}(l,i)$ denote the gating value node $i$ in layer $l$. We say that the gate to be \emph{transitioning} for input $s\in[n]$, and weight $\tg(m),m\in[d_{net}]$ if
 \begin{align}
 \left|\frac{\partial G_{x_s,\Tg_t}(l,i)}{\partial \tg(m)}\right|>\gamma,
 \end{align}
 and define a gate to be \emph{flipped} otherwise. Note that,
\begin{align}\label{eq:sensitivepath}
\begin{split}
&\partial_{m}A_{\Tg_t}(x_s,p)=\partial_{m}\Pi_{l=1}^{d-1} G_{x_s,\Tg_t}(l,p(l))\\
&=\sum_{l=1}^{d-1} \partial_{m} G_{x_s,\Tg_t}(l,p(l)) \left(\Pi_{l'\neq l} G_{x_s,\Tg_t}(l',p(l'))\right)
\end{split}
\end{align}

\textbf{Remark:}

i) As $\beta\uparrow\infty$, the soft-ReLU gate resembles the ReLU gate. Thus for a given input example $s$, the gates whose pre-activation inputs have a large absolute value will be close to either $0$ or $1$, and one can always find a high enough $\beta$ such that their sensitivity to $\tg(m)$ is less than $\gamma$.

ii) For an input examples $s,s'\in[n]$, if a path $p$ is active (even for one of the inputs), i.e., $A(x_s,p)\approx 1$, then none of the gates in the path will be sensitive, and hence the magnitude contribution of such as path to the summation in $\delta$ is close to $0$.

iii) For an input examples $s,s'\in[n]$, consider a non-active path, such that all gates close to $1$ except for one of the gates (i.e., the right hand side of \eqref{eq:sensitivepath} is non-zero), which is transitioning. Such paths will make a significant contribution to $\delta$ term. We call the set of such paths the sensitive sub-network.

Based on the above discussion one can say  that a DGN with adaptable gates (which includes standard DNN with ReLU gates), at initialisation, has two kinds of sub-networks for every input example i) the active sub-network comprised of path for which $A(x_s,p)=1$\footnote{or $A(x_s,p)$  is above a given threshold value in the case of soft gates} and ii) the sensitive sub-network which is formed by the set of paths that are sensitive for a given input.
}
\bibliographystyle{plainnat}
\bibliography{refs}

\onecolumn
\begin{center}
Appendix/Supplementary Material
\end{center}
\comment{
\textbf{Errata:} 

\textbf{Errors:} The following are the errors in the submission file ``icml2020.pdf'':
\begin{enumerate}
\item The references to Theorem $6.3$  (page $3$, $2^{nd}$-column from left, line number $144$, and page $4$, $2^{nd}$ column from left, line number $207$) should be read as Theorem $2.2$. 
\item In Lemma $2.3$, the displayed equations in the main text (page $4$, $2^{nd}$ column from left, lines $173-174$) contain an extra $``d"$ term in the right hand side:

(i) $\mathbb{E}_p\left[\lambda_0(s,s)\right]=\bar{\lambda}_{self}=d(pw)^{d-1}$ 

(ii) $\mathbb{E}_p\left[\lambda_0(s,s')\right]=\bar{\lambda}_{cross}= d(p^2w)^{d-1}$,

\item The reference to Assumption $1$ (page $4$, left column, line $194$) is incorrect, it should instead be read as Assumption $2$.

\item The text, \textbf{Active Sub-Network}, in page $7$, $2^{nd}$ column from the left, line $373$, should be read as \textbf{Sensitive Sub-Network}.
\end{enumerate} 

\textbf{Corrections (without editing main latex-source):} We are submitting the full paper including the supplementary material. The following are the corrections in the supplementary file ``icml2020$\_$supplementary.pdf'':
\begin{enumerate}
\item At the time of submission, incorrect references to Theorem $6.3$ occurred due to the use of the same label for Theorem $2.2$ in the main body and Theorem $6.3$ in the supplementary material. We have fixed this by removing the label in the supplementary material. 
\item In Lemma $2.3$, the displayed equations should be read as:

(i) $\mathbb{E}_p\left[\lambda_0(s,s)\right]=\bar{\lambda}_{self}=(pw)^{d-1}$ 

(ii) $\mathbb{E}_p\left[\lambda_0(s,s')\right]=\bar{\lambda}_{cross}= (p^2w)^{d-1}$,

this error has been fixed in the supplementary material in Lemma~\ref{lm:dgn-fra-restate}.
\end{enumerate} 

\textbf{Missing Information} 

The architecture used in \textbf{Experiment $6$} and \textbf{Experiment $7$} is: input layer $(32,32,3)$, followed by $(3,3)$ convolutional layers of channels $64,64,128,128$, followed by fully-connected layers of width $256,256,10$. 

\FloatBarrier
\begin{table}[h]
\centering
\begin{tabular}{|c|c|c|}\hline
Network
&Training Accuracy(mean of best)
&Test Accuracy(mean of best)\\\hline
Deep Linear Network
&0.5042
&0.3939\\\hline
Relu
&0.9987
&0.7132\\\hline

$\beta=$1, $\epsilon=$0.1
&0.9948
&0.5969\\\hline
$\beta=$2, $\epsilon=$0.1
&0.9961
&0.6401\\\hline
$\beta=$4, $\epsilon=$0.1
&0.9978
&0.6841\\\hline
$\beta=$8, $\epsilon=$0.1
&0.9985
&0.7100\\\hline
$\beta=$12, $\epsilon=$0.1
&0.9987
&0.7190\\\hline
$\beta=$16, $\epsilon=$0.1
&0.9988
&0.7203\\\hline
$\beta=$20, $\epsilon=$0.1
&0.9990
&0.7221\\\hline
$\beta=$24, $\epsilon=$0.1
&0.9991
&0.7230\\\hline
$\beta=$1, $\epsilon=$0.4
&0.9926
&0.5638\\\hline
$\beta=$2, $\epsilon=$0.4
&0.9944
&0.6374\\\hline
$\beta=$4, $\epsilon=$0.4
&0.9968
&0.6831\\\hline
$\beta=$8, $\epsilon=$0.4
&0.9980
&0.7115\\\hline
$\beta=$12, $\epsilon=$0.4
&0.9982
&0.7153\\\hline
$\beta=$16, $\epsilon=$0.4
&0.9986
&0.7143\\\hline
$\beta=$20, $\epsilon=$0.4
&0.9986
&0.7139\\\hline
$\beta=$24, $\epsilon=$0.4
&0.9989
&0.7099\\\hline
\end{tabular}
\caption{Shows performance of various networks for CIFAR 10 dataset. The results are averaged over $5$ runs. Here `mean of best' means the best performance obtained in the run.}
\label{tb:cifar}
\end{table}

\newpage
\begin{figure*}
\resizebox{\textwidth}{!}{
\begin{tabular}{cccc}
\includegraphics[scale=0.1]{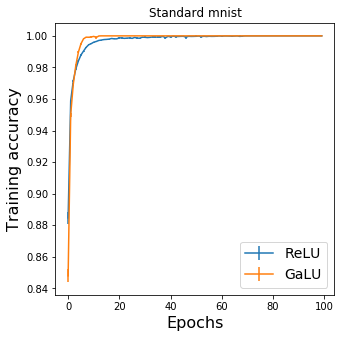}
&
\includegraphics[scale=0.1]{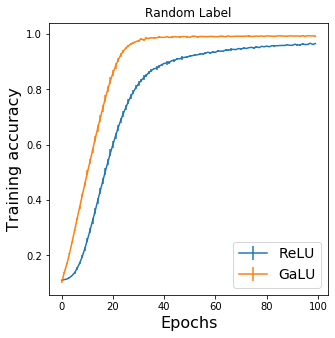}
&
\includegraphics[scale=0.1]{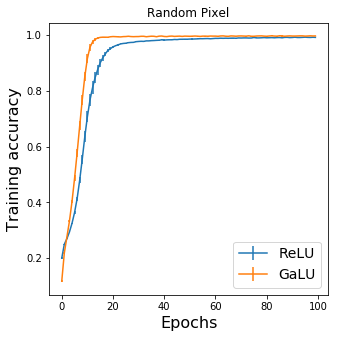}
&
\includegraphics[scale=0.1]{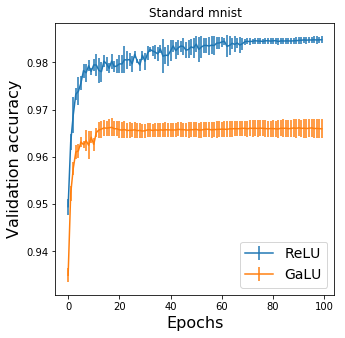}
	\end{tabular}
}
\caption{First three plots from the left show optimisation in ReLU and GaLU networks for standard MNIST, MNIST with random label and pixels  The right most plot shows generalisation of ReLU and GaLU networks in standard MNIST. Architecture we use is $d=7$, with layer widths from the first to last given by $512,512,256,256,128,64,10$ followed by a soft-max layer. In the case of GaLU there are two such network, wherein, one network is used to generate the gating values (whose weights are frozen) and the other network has the weights that are trained.}
\label{fig:galu-relu}
\end{figure*}

\newpage
}
\section{Paths}\label{sec:path}
\textbf{Vectorised Notation:} Given a dataset $(x_s,y_s)_{s=1}^n\in \R^{d_{in}}\times \R$, let data be represented as matrices $x\in\R^{d_{in}\times n}$ and $y\in \R^n$ with the convention that $x_s=x(\cdot,s)\in\R^{d_{in}}$ and $y_s=y(s)\in \R$. For the purpose of this section we follow the vectorised notation in \Cref{tb:dgnvector}.

\FloatBarrier
\begin{table}[h]
\centering
\begin{tabular}{|c|c|}\hline
Input layer & $x(s,i,0) =x(i,s)$ \\\hline
Pre-activation& $q_t(s,i,l)= {\Theta_t(l,\cdot,i)}^\top x_t(s,\cdot,l-1) $ \\\hline
Layer output & $x_t(s,i,l)= q_t(s,i,l) G_t(s,i,l)$ \\\hline
Final output & $\hat{y}_t(x)={\Theta_t(d,\cdot,1)}^\top x_t(s,\cdot,d-1)$\\\hline
\end{tabular}
\caption{A deep gated network in the vectorised form. $l=1,\ldots,d-1$ denote the intermediate layers.}
\label{tb:dgnvector}
\end{table}

The idea behind the ``path view'' is to regard the given neural network as multitude of connections from input to output.  We now describe the zeroth and first order terms in the language of paths.
\begin{definition}[Neural Path]
Let $\P=[d_{in}]\times [w]^{d-1}$ be a cross product of index sets. Define a path $p$ by $p\stackrel{def}=(p(0),p(1),\ldots,p(d-1))\in \P$, where $p(0)\in [d_{in}]$, and $p(l)\in[w],\forall l\in[d-1]$. 
\end{definition}

A path $p$ starts at an input node $p(0)$ goes through nodes $p(l)$ in layer $l\in[d-1]$ and finishes at the output node .

\begin{definition}\label{def:strength}[Strength]
Each path is also associated with a strength given by: $w_t(p)=\Pi_{l=1}^d \Theta_t(l,p(l-1),p(l))$
\end{definition}

\begin{definition}\label{def:activity}[Activation Level]
The activity of a path $p$ for input $s$ is given by: $A(s,p)=\Pi_{l=1}^d G(s,p(l),l)$
\end{definition}
In the case when $G\in \{0,1\}$ it also implies that $A\in \{0,1\}$.  

\begin{definition}\label{def:feature}[Neural Feature]
Given a gating pattern $\G_t$, define 
\begin{align}
\phi_{x_s,\G_t}(p)\stackrel{def}=x(p(0),s) A_{\G_t}(x_s,p),
\end{align}
and let $\phi_{x_s,\G_t}=(\phi_{x_s\G_t}(p),p\in [P])\in \R^P$ be the hidden feature corresponding to input $x_s$. Let $\Phi_{x,\G_t}=\left[\phi_{x_1,\G_t}| \ldots |\phi_{x_n,\G_t}\right]\in \R^{P\times n}$ be the feature matrix obtained by stacking the features $\phi_{x_s,\G_t}$ of inputs $x_s\in \R^{d_{in}}$ column-wise.
\end{definition}

\comment{
\textbf{Predicted} output of the network is given in terms of the paths by $\hat{y}_{t}(s)=\sum_{p\in P} x(p(0),s) A_{\G_t}(x_s,p) w_t(p)$, i.e., 
\begin{align}\label{eq:zeroth}
\hat{y}_{t}(s)=\Phi_{x,\G_t}^\top w_{t}
\end{align}


\textbf{Sub-networks:}  In DGNs similarity of two different inputs $x_s,x_{s'}\in \R^{d_{in}}, s,s' \in [n]$ depends on the overlap of path that are \emph{active} for both inputs. This can be seen by noting that $\phi_{x_s,\G_t}^\top \phi_{x_s,\G_t}=\sum_{i=1}^{d_{in}} x(i,s)x(i,s') \underset{p\rsa i}{\sum} A_{\G_t}(x_s,p) A_{\G_t}(x_{s'},p)$. Consider for instance a DGN whose gating values are in $\{0,1\}$, and say $n=2$, i.e., the dataset contains only two inputs $x_1,x_2\in \R^{d_{in}}$. From \eqref{eq:pathsim} it is clear that if inputs $x_1,x_2$ do not share any common \emph{active} paths throughout training, then they are \emph{orthogonal} to each other, because $A_{\G_t}(s,p) A_{\G_t}(s',p)=0, \forall p\in [P]$. This makes intuitive sense because in this case, it is as though there are two parallel networks (while the weights can be shared, the paths are not). Thus it is clear from the path based representation in \eqref{eq:zeroth} and \eqref{eq:pathsim}, the gating pattern $\G_t$ plays are crucial role in DGNs via the activation levels $A_{\G_t}(\cdot,\cdot)$.
The feature $\phi_{x_s,\G_t}$ in \Cref{def:feature} as well as the strength $w_t$ in \Cref{def:strength} are $P$-dimensional quantities. However, loosely speaking, the DGN has only as much \emph{degrees of freedom} as the number of trainable parameters (which we denote by $d_{net}$). 
\begin{definition}[Neural Tangent Features] The  $d_{net}\times n$ NTF matrix is defined as $\Psi_t(m,s)\stackrel{def}=\frac{\partial \hat{y}_{t}(x_s)}{\partial \theta(m)},m\in [d_{net}], s\in [n]$. 
\comment{
\begin{align}\label{eq:split}
\begin{split}
&\Psi_t(m,s) = \frac{\partial \hat{y}_t(x_s)}{\partial \theta(m)}\\
&=\frac{\partial }{\partial \theta(m)}\left(\sum_{p\in P} x(p(0),s) w_{t}(p) A_{\G_t}(s,p)\right),\\
&=\underbrace{\left(\sum_{p\in P} x(p(0),s) \frac{\partial w_{t}(p)}{\partial \theta(m)} A_{\G_t}(s,p)\right)}_{\text{sensitivity of strength}}\\
&\quad\quad\quad\quad\quad\quad\quad\quad+\\
&=\underbrace{\left(\sum_{p\in P} x(p(0),s) w_{t}(p) \frac{\partial A_{\G_t}(s,p)}{\partial \theta(m)}\right)}_{\text{sensitivity of activations}}
\end{split}
\end{align}

\begin{align}\label{eq:split}
\begin{split}
&\Psi_t(m,s) = \frac{\partial \hat{y}_t(x_s)}{\partial \theta(m)}\\
&=\frac{\partial }{\partial \theta(m)}\left(\sum_{p\in P} x(p(0),s) w_{t}(p) A_{\G_t}(s,p)\right),\\
&=\underbrace{\left(\sum_{p\in P} x(p(0),s) \frac{\partial w_{t}(p)}{\partial \theta(m)} A_{\G_t}(s,p)\right)}_{{\Psi^w_{t}(m,s)}}\\
&\quad\quad\quad\quad\quad\quad\quad\quad+\\
&=\underbrace{\left(\sum_{p\in P} x(p(0),s) w_{t}(p) \frac{\partial A_{\G_t}(s,p)}{\partial \theta(m)}\right)}_{{\Psi^a_t(m,s)}}
\end{split}
\end{align}
}
\comment{
\begin{align}\label{eq:split}
\begin{split}
\Psi_t(m,s) &= \underbrace{\left(\sum_{p\in P} x(p(0),s) \frac{\partial w_{t}(p)}{\partial \theta(m)} A_{\G_t}(s,p)\right)}_{{\Psi^w_{t}(m,s)}}\\
&\quad\quad\quad\quad\quad\quad\quad\quad+\\
&=\underbrace{\left(\sum_{p\in P} x(p(0),s) w_{t}(p) \frac{\partial A_{\G_t}(s,p)}{\partial \theta(m)}\right)}_{{\Psi^a_t(m,s)}}
\end{split}
}
\begin{align}\label{eq:split}
\begin{split}
\Psi_t(m,s) &= \Psi^w_{t}(m,s)+\Psi^a_{t}(m,s), \,\text{where}\\
\Psi^w_{t}(m,s)&={\left(\sum_{p\in P} x(p(0),s) \frac{\partial w_{t}(p)}{\partial \theta(m)} A_{\G_t}(s,p)\right)}\\
\Psi^a_{t}(m,s)&={\left(\sum_{p\in P} x(p(0),s) w_{t}(p) \frac{\partial A_{\G_t}(s,p)}{\partial \theta(m)}\right)}
\end{split}
\end{align}

\end{definition}

\textbf{Strength and Gate adaptation:} From \eqref{eq:split} it is clear that there are two \emph{atomic} components to the gradient of the output $\hat{y}_t(x_s)$ with respect to any trainable weight $\theta(m), m=1,\ldots, d_{net}$, namely \emph{neural tangent feature of strength} denoted by $\Psi^w_{t}\in \R^{d_{net}\times n}$ and \emph{neural tangent feature of activations} denoted by $\Psi^a_{t}\in \R^{d_{net}\times n}$. 

}

\comment{
\begin{definition}[Neural Tangent Features of Path Activations (NTFPA)]
\begin{align}
\varphi^a_{t,p}\stackrel{def}{=}\left(\frac{\partial w_{t}(p)}{\partial \theta(m)},m\in[d_{net}]\right)\in \R^{d_{net}},
\end{align}
\end{definition}

\begin{remark}
Let $\theta(m)$ belong to layer $l'\in [d-1]$, then 
\begin{align*}
\varphi^a_{t,p}(m)&=0, \forall p\bcancel{\rsa}\theta(m)
\end{align*}
\end{remark}

Using the sensitivity of strengths and activations at the level of resolution of paths, we now define the neural tangent feature (NTF) for the strengths and activations.

\begin{definition}\label{def:ntf}[Neural Tangent Features]
\begin{align}
\begin{split}
\Psi^w_t(m,s) &=\left(\sum_{p\in P} x(p(0),s) \frac{\partial w_{t}(p)}{\partial \theta(m)} A_{\G_t}(s,p)\right)\\
\Psi^a_t(m,s) &=\left(\sum_{p\in P} x(p(0),s) w_{t}(p) \frac{\partial A_{\G_t}(s,p)}{\partial \theta(m)}\right)\\
\Psi_t(m,s)&=\Psi^w_t(m,s)+ \Psi^a_t(m,s)
\end{split}
\end{align}
\end{definition}

\begin{definition}[Interaction Coefficient]
\begin{align*}
&\lambda^{s,s'}_t(i)\stackrel{def}=\underset{p_1,p_2\rsa\theta(m),i}{\sum_{p_1,p_2\in P:}}  \varphi_{t,p_1}(m)A(s,p_1)  \varphi_{t,p_2}(m) A(s',p_2)
\end{align*}
\end{definition}

\begin{lemma}
\begin{align*}
{K^w_t(s,s')}=\sum_{i=1}^{d_{in}} x(i,s)x(i,s') \lambda^{s,s'}_t(i)
\end{align*}
\end{lemma}

\begin{assumption}\label{assmp:main}\hfill
\begin{enumerate}
\item $\G_0$ is statistically independent of $\Theta_0$.
\item $\Theta_0\stackrel{iid}\sim Ber(\frac{1}{2})$ over the set $\{-\sigma,+\sigma\}$. 
\end{enumerate}
\end{assumption}

\textbf{Remarks on Assumption~\ref{assmp:main}}
In a standard DNN with ReLU activations, the activations and weights are not statistically independent because conditioned on the fact that a ReLU is \emph{on}, the incoming edges cannot be simultaneously all $-\sigma$. We side step this issue by the first condition in Assumption~\ref{assmp:main}, wherein, we assume that  gating is statistically independent of the weights $\Theta_0$. This clears the way to carry out the algebra of paths, which can be boiled down in simple words as the effect on weights in the direction of one path does not affect the contribution of any other path in expectation. This is captured in Lemma~\ref{lm:pathdot} below.
\begin{figure}
\centering
\includegraphics[scale=0.2]{mickey.png}
\caption{Shows that the incoming weights of a ReLU gate which are \emph{on} are not symmetrically distributed.}
\end{figure}

\begin{lemma}\label{lm:pathdot}
Let $\theta(m)$ be an arbitrary weight in layer $l'\in [d-1]$, under Assumption~\ref{assmp:init} we have for paths $p,p_1,p_2\rsa\theta(m), p_1\neq p_2$
\begin{align*}
\E{\varphi_{\Tb,p_1}(m)\varphi_{\Tb,p_2}(m)}= &0\\
\E{\varphi_{\Tb,p}(m)\varphi_{\Tb,p}(m)}= &\left(\frac{2\sigma^2}{w}\right)^{d-1}
\end{align*}
\end{lemma}
\begin{proof}
If $\theta$
Note that $\varphi_{\Tb,p}=\underset{l\neq l'}{\underset{l=1}{\overset{d}{\Pi}}} \Tb(l,p(l-1),p(l))$. Hence
\begin{align*}
&\E{\varphi_{\Tb,p_1}(m)\varphi_{\Tb,p_2}(m)}\\
&=\E{\underset{l\neq l'}{\underset{l=1}{\overset{d}{\Pi}}} \Bigg(\Tb(l,p_1(l-1),p_1(l))\Tb(l,p_2(l-1),p_2(l)) \Bigg)}\\
&=\underset{l\neq l'}{\underset{l=1}{\overset{d}{\Pi}}}\E{\Tb(l,p_1(l-1),p_1(l))\Tb(l,p_2(l-1),p_2(l))}
\end{align*}

Since $p_1\neq p_2$, in one of the layers $\tilde{l}\in[d-1],\tilde{l}\neq l'$ they do not pass through the same weight. Using this fact
\begin{align*}
&\E{\varphi_{\Tb,p_1}(m)\varphi_{\Tb,p_2}(m)}\\
&=\left(\underset{l\neq l',\tilde{l}}{\underset{l=1}{\overset{d}{\Pi}}}\E{\Tb(l,p_1(l-1),p_1(l))\Tb(l,p_2(l-1),p_2(l))}\right)\\
&\Bigg(\E{\Tb(\tilde{l},p_1(\tilde{l}-1),p_1(\tilde{l}))}\E{\Tb(\tilde{l},p_2(\tilde{l}-1),p_2(\tilde{l}))}\Bigg)\\
&=0
\end{align*}
\end{proof}

\begin{definition}[Path Similarity]
\begin{align*}
\mu^{s,s'}(i)=\sum_{m=1}^{d_{net}} \underset{p\rsa\theta(m)}{\sum_{p\in P: p(0)=i}}A(s,p) A(s',p)
\end{align*}

\end{definition}

\begin{lemma}
\begin{align*}
\mathbf{E}_{\Theta_0}\left[\lambda^{s,s'}_0(i)\right]=\sigma^{2(d-1)}\mu^{s,s'}(i)
\end{align*}
\end{lemma}

\begin{theorem}\label{th:dgnexp}
 Under Assumption~\ref{assmp:main}
 \begin{align*}
\mathbf{E}_{\Theta_0}\left[K_0(s,s')\right]=\sigma^{2(d-1)}\sum_{i=1}^{d_{in}}x(i,s) x(i,s')\mu^{s,s'}(i)
\end{align*}

\end{theorem}
\begin{proof}
See Appendix.
\end{proof}

\begin{theorem}\label{th:dgnvar}
 Under Assumption~\ref{assmp:main} $Var\left[K_0\right]\leq $
\end{theorem}

}

\subsection{Results in \Cref{sec:optimisation}}
\textbf{Stament and Proof of Lemma~\ref{lm:sigwire}}
\begin{lemma}[Signal vs Wire Decomposition]
Let $\kappa_t(s,s',i)\stackrel{def}=\underset{p_1,p_2\rsa i}{\sum_{p_1,p_2\in P:}} A_{\G_t}(x_s,p_1) A_{\G_t}(x_{s'},p_2) \ip{\varphi_{t,p_1}, \varphi_{t,p_2}}$. The Gram matrix $K_t$ is then given by 
\begin{align}\label{eq:ktalg}
{K_t(s,s')}=\sum_{i=1}^{d_{in}} x(i,s)x(i,s') \kappa_t(s,s',i)
\end{align}
\end{lemma}

\begin{proof}
Note that
\begin{align*}
\hat{y}_{t}(x_s)=\sum_{p\in\P}x(p(0),s) A(x_s,p) w_t(p)
\end{align*}
Differentiating with respect to any of the weights $\theta(m),m\in[d_{net}]$, we have
\begin{align*}
\frac{\partial \hat{y}_{t}(x_s)}{\partial \theta(m)}&=\frac{\partial \sum_{p\in\P}x(p(0),s) A(x_s,p) w_t(p)}{\partial \theta(m)}\\
\Psi_t(m,s)&=\sum_{p\in\P}x(p(0),s) A(x_s,p) \frac{\partial w_t(p)}{\partial \theta(m)}\\
&=\sum_{p\in\P}x(p(0),s) A(x_s,p) \varphi_{t,p}(m)
\end{align*}

Since, only the path strengths are changing, the Gram matrix $K_t$ is given by 
\begin{align*}
K_t(s,s')&={\Psi_t(\cdot,s)}^\top \Psi_t(\cdot,s')\\
&=\sum_{m=1}^{d_{net}} \Psi_t(m,s) \Psi_t(m,s')\\
&=\sum_{m=1}^{d_{net}} \left(\sum_{p_1\in\P}x(p_1(0),s) A(x_s,p_1) \varphi_{t,p_1}(m)\right)\left(\sum_{p_2\in\P}x(p_2(0),s') A(x_{s'},p_2) \varphi_{t,p_2}(m)\right)\\
&=\sum_{m=1}^{d_{net}} \underset{p_1,p_2\rsa\theta(m)}{\sum_{p_1,p_2\in P:}} x(p_1(0),s) A(x_s,p_1)x(p_2(0),s') A(x_{s'},p_2) \varphi_{t,p_1}(m) \varphi_{t,p_2}(m)\\
&=\sum_{i=1}^{d_{in}}\underset{p_1,p_2\rsa i}{\sum_{p_1,p_2\in P:}} x(p_1(0),s) A(x_s,p_1)x(p_2(0),s') A(x_{s'},p_2) \ip{\varphi_{t,p_1}, \varphi_{t,p_2}}\\
&=\sum_{i=1}^{d_{in}}\underset{p_1,p_2\rsa i}{\sum_{p_1,p_2\in P:}} x(i,s) A(x_s,p_1)x(i,s') A(x_{s'},p_2) \ip{\varphi_{t,p_1}, \varphi_{t,p_2}}\\
&=\sum_{i=1}^{d_{in}} x(i,s)x(i,s') \underset{p_1,p_2\rsa i}{\sum_{p_1,p_2\in P:}} A(x_s,p_1) A(x_{s'},p_2) \ip{\varphi_{t,p_1}, \varphi_{t,p_2}}
\end{align*}
\end{proof}

\textbf{Statement and Proof of Lemma~\ref{lm:pathdot}}
\begin{lemma}
Under Assumption~\ref{assmp:mainone}, for paths $p,p_1,p_2\in \P, p_1\neq p_2$, at initialisation we have (i) $\E{\ip{\varphi_{0,p_1}, \varphi_{0,p_2}}}= 0$, (ii) ${\ip{\varphi_{0,p}, \varphi_{0,p}}}= d\sigma^{2(d-1)}$
\end{lemma}

\begin{proof}
\begin{align*}
\ip{\varphi_{t,p_1}, \varphi_{t,p_2}}= \sum_{m=1}^{d_{net}} \varphi_{t,p_1}(m)\varphi_{t,p_2}(m)
\end{align*}
Let $\theta(m),m\in[d_{net}]$ be any weight such that $p\rsa \theta(m)$, and w.l.o.g let $\theta(m)$ belong to layer $l'\in[d]$. 
If either $p_1\bcancel{\rsa}\theta(m)$ or $p_2\bcancel{\rsa}\theta(m)$, then it follows that $\varphi_{t,p_1}(m)\varphi_{t,p_2}(m)=0$. In the case when $p_1,p_2\rsa\theta(m)$, we have
\begin{align*}
&\E{\varphi_{0,p_1}(m)\varphi_{0,p_2}(m)}\\
&=\E{\underset{l\neq l'}{\underset{l=1}{\overset{d}{\Pi}}} \Bigg(\Tb_0(l,p_1(l-1),p_1(l))\Tb_0(l,p_2(l-1),p_2(l)) \Bigg)}\\
&=\underset{l\neq l'}{\underset{l=1}{\overset{d}{\Pi}}}\E{\Tb_0(l,p_1(l-1),p_1(l))\Tb_0(l,p_2(l-1),p_2(l))}
\end{align*}
where the $\E{\cdot}$ moved inside the product because at initialisation the weights (of different layers) are independent of each other.
Since $p_1\neq p_2$, in one of the layers $\tilde{l}\in[d-1],\tilde{l}\neq l'$ they do not pass through the same weight, i.e., $\Tb_0(\tilde{l},p_1(\tilde{l}-1),p_1(\tilde{l}))$ and $\Tb_0(\tilde{l},p_2(\tilde{l}-1),p_2(\tilde{l}))$ are distinct weights. Using this fact
\begin{align*}
&\E{\varphi_{0,p_1}(m)\varphi_{0,p_2}(m)}\\
&=\underset{l\neq l',\tilde{l}}{\underset{l=1}{\overset{d}{\Pi}}}\E{\Tb_0(l,p_1(l-1),p_1(l))\Tb_0(l,p_2(l-1),p_2(l))}\\
&=\E{\Tb_0(\tilde{l},p_1(\tilde{l}-1),p_1(\tilde{l}))}\E{\Tb_0(\tilde{l},p_2(\tilde{l}-1),p_2(\tilde{l}))}\\
&=0
\end{align*}

The proof of (ii) is complete by noting that $\sum_{m=1}^{d_{net}} \varphi_{t,p}(m)\varphi_{t,p}(m)$ has $d$ non-zero terms for a single path $p$ and at initialisation we have 
\begin{align*}
&{\varphi_{0,p}(m)\varphi_{0,p}(m)}\\
&={\underset{l\neq l'}{\underset{l=1}{\overset{d}{\Pi}}} \Tb^2_0(l,p(l-1),p(l))}\\
&=\sigma^{2(d-1)}
\end{align*}
\end{proof}

\textbf{Statement and Proof of Theorem~\ref{th:dgnexp}}
\begin{theorem}[DIP in DGN]
Under Assumption~\ref{assmp:mainone}, ~\ref{assmp:maintwo}, and $\frac{4d}{w^2}<1$ it follows that
 \begin{align*}
&\E{K_0}=d\sigma^{2(d-1)}(x^\top x \odot \lambda_0)\\
&Var\left[K_0\right]\leq O\left(d^2_{in}\sigma^{4(d-1)}\max\{d^2w^{2(d-2)+1}, d^3w^{2(d-2)}\}\right)
\end{align*}
\end{theorem}

\begin{proof}
The first of the above two claims follow from the algebraic expression for $K_t$ and Lemma~\ref{lm:pathdot}. We now look at the variance calculation. The idea is that we expand  $Var\left[K_0(s,s')\right]=\E{K_0(s,s')^2} -\E{K_0(s,s')}^2$ and identify the the terms which cancel due to subtraction and then bound the rest of the terms.

Let $\theta(m)$ belong to layer $l'(m)$, then 
\begin{align}\label{eq:kexpect}
\E{K_0(s,s')}&=\sum_{m=1}^{d_{net}}\E{\left(\sum_{p_1 \in P}x(p_1(0),s)A(s,p_1)\frac{\partial w_{\Tb}(p_1)}{\partial \theta(m)}\right)\left(\sum_{p_2\in P}x(p_2(0),s)A(s',p_2)\frac{\partial w_{\Tb}(p_2)}{\partial \theta(m)}\right)}\nn\\
&=\sum_{m=1}^{d_{net}}\E{\sum_{p_1,p_2\in P}x(p_1(0),s)A(s,p_1)\frac{\partial w_{\Tb}(p_1)}{\partial \theta(m)}x(p_2(0),s')A(s',p_2)\frac{\partial w_{\Tb}(p_2)}{\partial \theta(m)}}\nn\\
&=\sum_{m=1}^{d_{net}}\underset{p_1,p_2\rsa\theta(m)}{\sum_{p_1,p_2\in P}}x(p_1(0),s)A(s,p_1)x(p_2(0),s')A(s',p_2) \E{\underset{l\neq l'(m)}{\underset{l=1}{\overset{d-1}{\Pi}}} \Tb_0(l,p_1(l-1),p_1(l)) \Tb_0(l,p_2(l-1),p_2(l))}\nn\\
&\stackrel{(a)}=\sum_{m=1}^{d_{net}}\underset{p_1,p_2\rsa\theta(m)}{\sum_{p_1,p_2\in P}}x(p_1(0),s)A(s,p_1)x(p_2(0),s')A(s',p_2) \underset{l\neq l'(m)}{\underset{l=1}{\overset{d-1}{\Pi}}} \E{\Tb_0(l,p_1(l-1),p_1(l)) \Tb_0(l,p_2(l-1),p_2(l))}
\end{align}
where $(a)$ follows from the fact that at initialisation the layer weights are independent of each other. Note that the right hand side of \eqref{eq:kexpect} only terms with $p_1=p_2$ will survive the expectation.

In the expression in \eqref{eq:kexpectsquare} note that $p_1=p_2$ and $p_3=p_4$.
\begin{align}\label{eq:kexpectsquare}
&\E{K_0(s,s')}^2=\nn\\
&\left(\sum_{m=1}^{d_{net}}\underset{p_1,p_2\rsa\theta(m)}{\sum_{p_1,p_2\in P}}x(p_1(0),s)A(s,p_1)x(p_2(0),s')A(s',p_2) \underset{l\neq l'(m)}{\underset{l=1}{\overset{d-1}{\Pi}}} \E{\Tb_0(l,p_1(l-1),p_1(l)) \Tb_0(l,p_2(l-1),p_2(l))}\right)\nn\\
&\left(\sum_{m'=1}^{d_{net}}\underset{p_3,p_4\rsa\theta(m')}{\sum_{p_3,p_4\in P}}x(p_3(0),s)A(s,p_3)x(p_4(0),s')A(s',p_4) \underset{l\neq l'(m')}{\underset{l=1}{\overset{d-1}{\Pi}}} \E{\Tb_0(l,p_3(l-1),p_3(l)) \Tb_0(l,p_4(l-1),p_4(l))}\right)\nn\\
&=\nn\\
&\sum_{m,m'=1}^{d_{net}}\underset{p_3,p_4\rsa\theta(m')}{\underset{p_1,p_2\rsa\theta(m)}{\sum_{p_1,p_2,p_3,p_4\in P}}}\Bigg[\bigg(x(p_1(0),s)A(s,p_1)x(p_2(0),s')A(s',p_2)x(p_3(0),s)A(s,p_3)x(p_4(0),s')A(s',p_4)\bigg)\nn\\
&\bigg( \underset{l\neq l'(m)} {\underset{l\neq l'(m')}{\underset{l=1}{\overset{d-1}{\Pi}}}} \E{\Tb_0(l,p_1(l-1),p_1(l)) \Tb_0(l,p_2(l-1),p_2(l))}\E{\Tb_0(l,p_3(l-1),p_3(l)) \Tb_0(l,p_4(l-1),p_4(l))} \bigg)\nn\\
&\bigg( \E{\Tb_0(l,p_1(l'(m')-1),p_1(l'(m'))) \Tb_0(l,p_2(l'(m')-1),p_2(l'(m')))}\bigg)\nn\\
&\bigg(\E{\Tb_0(l,p_3(l'(m)-1),p_3(l'(m))) \Tb_0(l,p_4(l'(m)-1),p_4(l'(m)))} \bigg)\Bigg]\nn\\
\end{align}

In the expression in \eqref{eq:ksquareexpect}, paths $p_1,p_2,p_3,p_4$ do not have constraints, and can be distinct.
\begin{align}\label{eq:ksquareexpect}
&\E{K^2_0(s,s')}=\nn\\
&\sum_{m,m'=1}^{d_{net}}\underset{p_3,p_4\rsa\theta(m')}{\underset{p_1,p_2\rsa\theta(m)}{\sum_{p_1,p_2,p_3,p_4\in P}}}\Bigg[\bigg(x(p_1(0),s)A(s,p_1)x(p_2(0),s')A(s',p_2)x(p_3(0),s)A(s,p_3)x(p_4(0),s')A(s',p_4)\bigg)\nn\\
&\bigg( \underset{l\neq l'(m)} {\underset{l\neq l'(m')}{\underset{l=1}{\overset{d-1}{\Pi}}}} \E{\Tb_0(l,p_1(l-1),p_1(l)) \Tb_0(l,p_2(l-1),p_2(l))\Tb_0(l,p_3(l-1),p_3(l)) \Tb_0(l,p_4(l-1),p_4(l))} \bigg)\nn\\
&\bigg( \E{\Tb_0(l,p_1(l'(m')-1),p_1(l'(m'))) \Tb_0(l,p_2(l'(m')-1),p_2(l'(m')))}\bigg)\nn\\
&\bigg(\E{\Tb_0(l,p_3(l'(m)-1),p_3(l'(m))) \Tb_0(l,p_4(l'(m)-1),p_4(l'(m)))} \bigg)\Bigg]\nn\\
\end{align}

We now state the following facts/observations.

$\bullet$ \textbf{Fact 1:} Any term that survives the expectation (i.e., does not become $0$) and participates in \eqref{eq:ksquareexpect} is of the form $\sigma^{4(d-1)}\big(x(p_1(0),s)A(s,p_1)x(p_2(0),s')A(s',p_2)x(p_3(0),s)A(s,p_3)x(p_4(0),s')A(s',p_4)\big)$, where $p_1,p_2,p_3,p_4$ are free variables, and participates in \eqref{eq:kexpectsquare} is of the form $\sigma^{4(d-1)}\big(x(p_1(0),s)A(s,p_1)x(p_2(0),s')A(s',p_2)x(p_3(0),s)A(s,p_3)x(p_4(0),s')A(s',p_4)\big)$, where $p_1=p_2,p_3=p_4$.

$\bullet$ \textbf{Fact 2:} The number of paths through a particular weight $\theta(m)$ in one of the middle layers is $d_{in}w^{d-3}$, and the number of paths through a particular weight $\theta(m)$ in either the first or the last layer is $d_{in}w^{d-2}$ .

$\bullet$ \textbf{Fact 3:} Let $\P'$ be an arbitrary set of paths constrained to pass through some set of weights. Let $\P''$ be the set of paths obtained by adding an additional constraint that the paths also should pass through a particular weight say $\theta(m)$. Now, if $\theta(m)$ belongs to :

$1.$ a middle layer, then $|\P''|=\frac{|\P'|}{w^2}$.

$2.$ the first layer or the last layer, then $|\P''|=\frac{|\P'|}{w}$.

$\bullet$ \textbf{Fact 4:} For any $p_1,p_2,p_3,p_4$ combination that survives the expectation in \eqref{eq:ksquareexpect} can be written as 

\begin{align*}
&\bigg( \underset{l\neq l'(m)} {\underset{l\neq l'(m')}{\underset{l=1}{\overset{d-1}{\Pi}}}} \E{\Tb_0(l,p_1(l-1),p_1(l)) \Tb_0(l,p_2(l-1),p_2(l))\Tb_0(l,p_3(l-1),p_3(l)) \Tb_0(l,p_4(l-1),p_4(l))} \bigg)\nn\\
&\bigg( \E{\Tb_0(l,p_1(l'(m')-1),p_1(l'(m'))) \Tb_0(l,p_2(l'(m')-1),p_2(l'(m')))}\bigg)\nn\\
&\bigg(\E{\Tb_0(l,p_3(l'(m)-1),p_3(l'(m))) \Tb_0(l,p_4(l'(m)-1),p_4(l'(m)))} \bigg)\Bigg]\nn\\
&=\\
&\bigg( \underset{l\neq l'(m)} {\underset{l\neq l'(m')}{\underset{l=1}{\overset{d-1}{\Pi}}}} \Tb^2_0(l,\rho_a(l-1),\rho_a(l)) \Tb^2_0(l,\rho_b(l-1),\rho_b(l)) \bigg)\nn\\
&\bigg( \Tb^2_0(l,\rho_a(l'(m')-1),\rho_a(l'(m')))\bigg)\nn\\
&\bigg({\Tb^2_0(l,\rho_b(l'(m)-1),\rho_b(l'(m)))} \bigg),
\end{align*}

where $\rho_a\rsa \theta(m)$ and $\rho_b\rsa \theta(m')$ are what we call as \emph{base} (case) paths.

$\bullet$ \textbf{Fact 5:} For any given base paths $\rho_a$ and $\rho_b$ there could be multiple assignments possible for $p_1,p_2,p_3,p_4$.

$\bullet$ \textbf{Fact 6:}  Terms in \eqref{eq:ksquareexpect}, wherein, the base case is generated as $p_1=p_2=\rho_a$ and $p_3=p_4=\rho_b$ (or $p_1=p_2=\rho_b$ and $p_3=p_4=\rho_a$), get cancelled with the corresponding terms in \eqref{eq:kexpectsquare}.

$\bullet$ \textbf{Fact 7:}  When the bases paths $\rho_a$ and $\rho_b$ do not intersect (i.e., do not pass through the same weight in any one of the layers), the only possible assignment is $p_1=p_2=\rho_a$ and $p_3=p_4=\rho_b$ (or $p_1=p_2=\rho_b$ and $p_3=p_4=\rho_a$), and such terms are common in \eqref{eq:ksquareexpect} and \eqref{eq:kexpectsquare}, and hence do not show up in the variance term.

$\bullet$ \textbf{Fact 7:} Let base paths $\rho_a$ and $\rho_b$ cross at layer $l_1, \ldots, l_k, k \in [d-1]$, and let $\rho_a=(\rho_a(1),\ldots,\rho_a(k+1))$ where $\rho_a(1)$ is a sub-path string from layer $1$ to $l_1$, and $\rho_a(2)$ is the sub-path string from layer $l_1+1$ to $l_2$ and so on, and $\rho_a(k+1)$ is the sub-path string from layer $l_k+1$ to the output node. Then the set of paths that can occur in $\E{K_0(s,s')^2}$ are of the form:
\begin{enumerate}
\item $p_1=p_2=\rho_a, p_3=p_4=\rho_b$ (or $p_1=p_2=\rho_b, p_3=p_4=\rho_a$) which get cancelled in the $\E{K_0(s,s')}^2$ term.
\item $p_1=\rho_a$, $p_3=\rho_b$, $p_2=(\rho_b(1),\rho_a(2),\rho_a(3),\ldots,\rho_a(k+1))$, $p_4=(\rho_a(1),\rho_b(2),\rho_b(3),\ldots,\rho_b(k+1))$, which are obtained by splicing the base paths in various combinations. Note that for such spliced paths $p_1\neq p_2$ and $p_3\neq p_4$ and hence do not occur in the expression for $\E{K_0(s,s')}^2$ in \eqref{eq:kexpectsquare}.
\end{enumerate}

$\bullet$ \textbf{Fact 8:} For $k$ crossings of the base paths there are $4^{k+1}$ splicings possible, and those many terms are extra in the $\E{K_0(s,s')^2}$ calculation in \eqref{eq:ksquareexpect} comparison to the $\E{K_0(s,s')}^2$ calculation. We now enumerate cases of possible crossings, and reason out the magnitude of their contribution to the variance term using the \textbf{Fact 1} to \textbf{Fact 8}.

\textbf{Case $1$} $k=1$ crossing in either first or last layer. There are $2w$ weights in the first and the last layer, and the number of base path combinations is $w^{d-2}\times w^{d-2}$, and for each of these cases, $m,m'$ could take $O(d^2)$ possible values. And the multiplication of the weights themselves contribute to $\sigma^{4(d-1)}$. Putting them together we have
\begin{align*}
d^2_{in}\sigma^{4(d-1)}\times (2w)\times d^2\times (w^{d-2}\times w^{d-2})\times 4^2 = 32d^2_{in}\sigma^{4(d-1)}d^2 w^{2(d-2)+1}
\end{align*}

\textbf{Case $2$} $k=1$ crossing in one of the middle layers. There are $w^2(d-2)$ weights in the first and the last layer, and the number of base path combinations is $w^{d-3}\times w^{d-3}$, and for each of these cases, $m,m'$ could take $O(d^2)$ possible values. And the multiplication of the weights themselves contribute to $\sigma^{4(d-1)}$. Putting them together we have
\begin{align*}
d^2_{in}\sigma^{4(d-1)}\times w^2(d-2)\times d^2\times (w^{d-3}\times w^{d-3})\times 4^2\leq 16d^2_{in}\sigma^{4(d-1)} d^3 w^{2(d-3)}
\end{align*}

\textbf{Case $3$} $k=2$ crossings one in the first layer and other in the last layer. This case can be covered using Case $1$ and then further restricting that the base paths should also in the other layer. So, we have
\begin{align*}
32d^2_{in}\sigma^{4(d-1)}d^2 w^{2(d-2)+1} \times \underbrace{w}_{\text{possible weights in other layer}} \times \underbrace{w^{-1}\times w^{-1}}_{\text{reduction in paths due to additional restriction}} \times 4 = (32d^2_{in}\sigma^{4(d-1)}d^2 w^{2(d-2)+1})\times (4w^{-1}),
\end{align*}
where the $4$ is for the $4$ extra possible ways of splicing the base paths.

\textbf{Case $4$} $k=2$ crossings first one in the first layer or the last layer, and the second one in the middle layer. This can be obtained by looking at the Case $1$ and then adding the further restriction that the base paths should cross each other in the middle layer. 
\begin{align*}
32d^2_{in}\sigma^{4(d-1)}d^2 w^{2(d-2)+1}\times w^2(d-2) \times (w^{-2}w^{-2}) \times 4= (32d^2_{in}\sigma^{4(d-1)}d^2 w^{2(d-2)+1} )\times (4dw^{-2}) 
\end{align*}

\textbf{Case $5$} $k=2$ crossings in the middle layer. This can be obtained by taking Case $2$ and then adding the further restriction that the base paths should cross each other in the middle layer. 
\begin{align*}
16d^2_{in}\sigma^{4(d-1)} d^3 w^{2(d-3)}\times w^2(d-2) w^{-2}w^{-2}\times 4\leq (16d^2_{in}\sigma^{4(d-1)} d^3 w^{2(d-3)}) \times (4dw^{-2})
\end{align*}

\textbf{Case $6$} $k=3$ crossings first one in the first layer or the last layer, and the other two in the middle layers. This can be obtained by considering Case $4$ and then adding the further restriction that the base paths should cross each other in the middle layer. 
\begin{align*}
(32d^2_{in}\sigma^{4(d-1)}d^2 w^{2(d-2)+1} )\times (4dw^{-2}) \times (4dw^{-2}) 
\end{align*}

\textbf{Case $7$} $k=3$ crossings first two in the first and last layers and the third one in the middle layers. This can be obtained by considering Case $3$ and then adding the further restriction that the base paths should cross each other in the middle layer. 

\begin{align*}
(32d^2_{in}\sigma^{4(d-1)}d^2 w^{2(d-2)+1})\times (4w^{-1})\times (4dw^{-2}) 
\end{align*}

\textbf{Case $8$} $k=3$ crossings in the middle layer. This can be obtained by considering Case $5$ and then adding the further restriction that the base paths should cross each other in the middle layer. 
\begin{align*}
 (16d^2_{in}\sigma^{4(d-1)} d^3 w^{2(d-3)}) \times (4dw^{-2})\times (4dw^{-2}) 
\end{align*}

The cases can be extended in a similar way, increasing the number of crossings.  Now, assuming $\frac{4d}{w^2}<1$, the bounds in the various terms can be lumped together as below:

$\bullet$ We can add the bounds for Case $1$, Case $4$, Case $6$ and other cases obtained by adding more crossings (one at a time) in the middle layer to Case $6$. This gives rise to a term which is upper bounded by 
\begin{align*}
d^2_{in}\sigma^{4(d-1)}d^2w^{2(d-2)+1}\left(\frac{1}{1-4dw^{-2}}\right)
\end{align*}

$\bullet$ We can add the bounds for Case $3$, Case $7$ and other cases obtained by adding more crossings (one at a time) in the middle layer to Case $6$. This gives rise to a term which is upper bounded by 
\begin{align*}
d^2_{in}\sigma^{4(d-1)}d^3w^{2(d-2)} \left(\frac{1}{1-4dw^{-2}}\right)
\end{align*}

$\bullet$ We can add the bounds for Case $2$, Case $5$, Case $8$ and other cases obtained by adding more crossings (one at a time) in the middle layer to Case $6$. This gives rise to a term which is upper bounded by 
\begin{align*}
d^2_{in}\sigma^{4(d-1)}d^2w^{2(d-2)} \left(\frac{1}{1-4dw^{-2}}\right)
\end{align*}

Putting together we have the variance to be bounded by 
\begin{align*}
Cd^2_{in}\sigma^{4(d-1)}\max\{d^2w^{2(d-2)+1}, d^3w^{2(d-2)}\},
\end{align*}
for some constant $C>0$.
\end{proof}

\textbf{Statement and Proof of Lemma~\ref{lm:dgn-fra}}
\begin{lemma}
 Under Assumption~\ref{assmp:mainone},~\ref{assmp:maintwo} and gates sampled iid $Ber(\mu)$, we have, $\forall s,s'\in[n]$

(i) $\mathbb{E}_p\left[\lambda_0(s,s)\right]=\bar{\lambda}_{self}=(\mu w)^{d-1}$

ii) $\mathbb{E}_p\left[\lambda_0(s,s')\right]=\bar{\lambda}_{cross}= (\mu^2w)^{d-1}$
\end{lemma}

\begin{proof}
The proof of (i) follows by noting that the average number of gates that are \emph{on} in each layer is $(\mu w)$, and there are $(\mu w)^{d-1}$ paths starting from a given input node $i\in[d_{in}]$ and ending at the output node. The proof of (ii) follow by noting that on an average $\mu^2w$ gates overlap per layer for two different inputs.
\end{proof}

\begin{lemma} 
Under Assumptions~\ref{assmp:mainone},~\ref{assmp:maintwo}, in soft-GaLU networks we have: (i) $\E{K_0}=\E{K^w_0}+\E{K^a_0}$, 
 (ii) $\E{K^w_0}=\sigma^{2(d-1)} (x^\top x)\odot \lambda$,  (iii) $\E{K^a_0}=\sigma^{2d}  (x^\top x)\odot \delta$
\comment{where (with $\partial_{m}=\partial_{\tg(m)}$ in the below), we have
\begin{align*}
\mu_w^{s,s'}(i)\stackrel{def}=&\sum_{m=1}^{d_{net}} \underset{p\rsa\theta(m)}{\sum_{p\in P: p(0)=i}}A_{\Tg_0}(x_s,p) A_{\Tg_0}(x_{s'},p)\\
\mu_a^{s,s'}(i)\stackrel{def}=&\sum_{m=1}^{d_{net}} \underset{}{\sum_{p\in P: p(0)=i}} \partial_{m}A_{\Tg}(x_s,p) \partial_{m}A_{\Tg}(x_{s'},p)
\end{align*}
}
\comment{
\begin{align*}
&\E{K^a_0(s,s')}=\left( \frac{2\sigma^2}{w}\right)^{d} \Big(\sum_{i=1}^{d_{in}}x(i,s) x(i,s') \\
&\sum_{m=1}^{d_{gnet}}\underset{p\rsa i}{{\sum_{p\in \P:} }} \partial_{\tg(m)}A_{\Tg}(s,p) \partial_{\tg(m)}A_{\Tg}(s',p)\Big)\\
\end{align*}
}
\end{lemma}

\begin{proof}
Follows from Lemma~\ref{lm:pathdot}, and noting that $\Tg_0$ and $\Tw_0$ are iid.
\end{proof}

\section{Deep Linear Networks}

In this case, $G(s,l,i)=1,\forall s\in[n],i\in[w],l\in[d-1]$. Note that all the paths are always active irrespective of which input is presented to the DLN. We can define the effective weight that multiplies each of the input dimensions as 
\begin{align}
\eta_{t}(i)\stackrel{def}= \sum_{p\in P: p(0)=i} w_{t}(p), i\in [d_{in}]
\end{align}
Using the above definition of $\eta=(\eta(i),i\in[d_{in}])\in \R^{d_{in}}$, the hidden feature representation can be simplified as 
\begin{align}
\hat{y}_t&=\Phi^\top_{x,1_{\dagger}} w_{t} \\&=x^\top \eta_{t}
\end{align}
 Thus it is clear that the DLN does not provide any high dimensional feature representation and the input features are retained as such. All that the depth adds is just a non-linear re-parameterisation of the weights. It also follows that $\lambda_0(s,s')=w^{d-1},\forall ,s,s'\in [n]$.

\begin{corollary}\label{th:dln} Under Assumption~\ref{assmp:maintwo}, for a DLN with $d_{in}=1$, and dataset with $n=1$ we have, \begin{align} \mathbf{E}_{\Tb}\left[K_0\right]=d(w\sigma^2)^{(d-1)}\end{align}
\end{corollary}

\textbf{Experiment 8:} We consider a dataset with $n=1$ and $(x,y)=(1,1)$, i.e., $d_{in}=1$, let $w=100$ and look at various value of depth namely  $d=2,4,6,8,10$. We set $\sigma=\sqrt{\frac{1}{w}}$ and the weights are drawn according to Assumption~\ref{assmp:mainone}. We set the learning rate to be $\alpha=\frac{0.1}{d}$, and for this setting we expect the error dynamics to be the following $\frac{e^2_{t+1}}{e^2_t}=0.81$.
\comment{
\begin{align*}
e_{t+1}&=e_t-\frac{0.1}{d}d(w\sigma^2)^{2(d-1)}e_t\\
&=0.9^te_t
\end{align*}
}
The results are shown in \Cref{fig:dln}. We observe that irrespective of the depth the error dynamics is similar (since $\alpha=\frac{0.1}{d}$ ). However, we observe faster (in comparison to the ideal rate of $0.81$) convergence of error to zero since the magnitude of $K_t$ increases with time (see \Cref{fig:dln}).

\FloatBarrier
\begin{figure*}[h]
\resizebox{\textwidth}{!}{
\begin{tabular}{cccc}
\includegraphics[scale=0.4]{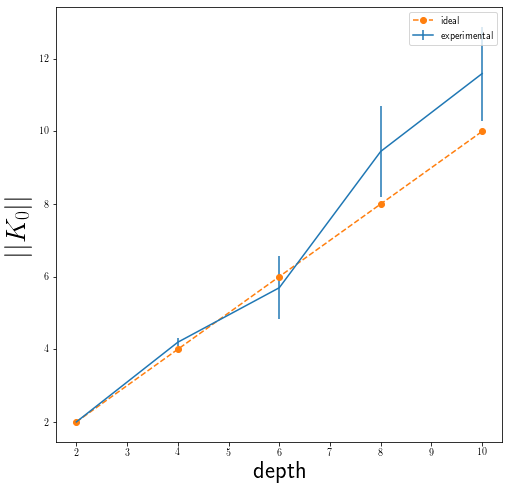}
&
\includegraphics[scale=0.4]{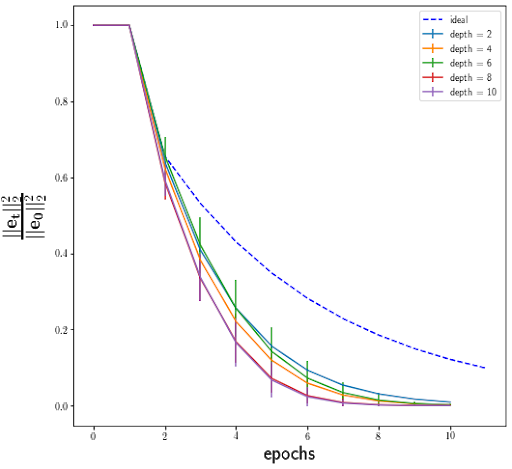}
&
\includegraphics[scale=0.4]{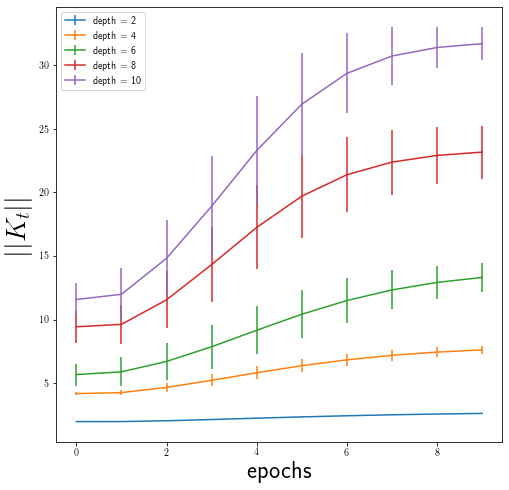}
&
\includegraphics[scale=0.4]{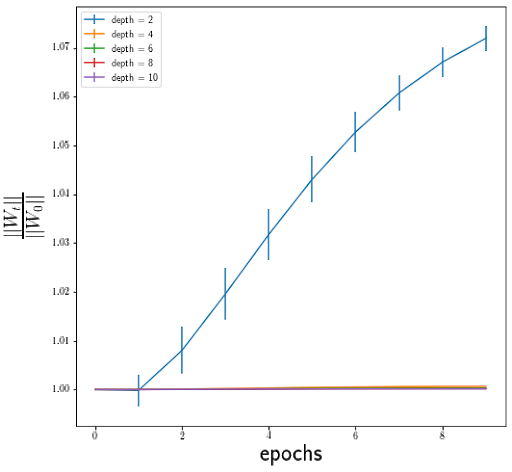}
\end{tabular}
}
\caption{In all the plots $d_{in}=1, n=1, w=100,\sigma^2=\frac{1}{w}$ averaged over $5$ runs. The left most plot shows $K_0$ as a function of depth. The second from left plot shows the convergence rate. The third plot from left shows the growth of $K_t$ over the course of training, and the right most plot shows the growth of weights ($L_2$-norm) with respect to time.}
\label{fig:dln}
\end{figure*}

\comment{\section{DGN-FRG}
\textbf{Effect of $p$} is shown in \Cref{fig:peff}. For $w=100$, we observe that for  the e.c.d.f gets better as the value of $p$ reduces till $p=0.3$, after which it starts degrading. This is due to the fact that the variance gets worse with $\frac{1}p$ (since $\sigma=\sqrt{\frac{1}{pw}}$. It can be seen that for $w=50$ the variance is more and hence the e.c.d.f gets better as we reduce $p$ only till $p=0.4$, after which it starts to degrade.

\FloatBarrier
\begin{figure*}[h]
\resizebox{\columnwidth}{!}{
\begin{tabular}{cc}
\includegraphics[scale=0.4]{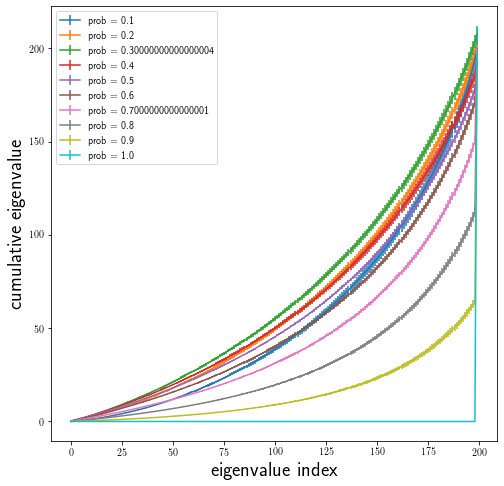}
&
\includegraphics[scale=0.4]{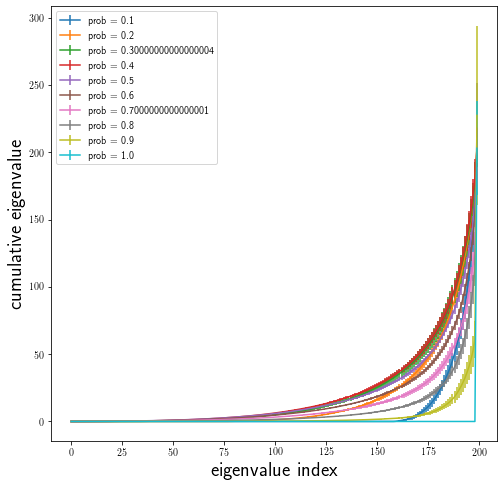}
\end{tabular}
}
\caption{Shows e.c.d.f for various values of $p$.}
\label{fig:peff}
\end{figure*}}

\textbf{Statement and Proof of Lemma~\ref{lm:invariance}}
\begin{lemma}
At $t=0$, under Assumptions~\ref{assmp:mainone},\ref{assmp:maintwo}, convolutional layers with global average pooling at the end causes translational invariance.
\begin{align*}
&\E{x_s(L,1)x_{s'}(L,1)}\\&=\frac{\sigma^{2(d-1)}}{d^2_{in}}\sum_{k=1}^{\hat{B}} \sum_{p_1,p_2\in b_k}  \Big( x(p_1(0),s) A(x_s,p_1)\\
&\quad\quad \quad\quad \quad\quad x(p_2(0),s') A(x_{s'},p_2) \Big)
\end{align*}
\end{lemma}

\begin{proof}
\begin{align*}
\E{x_s(L,1)x_{s'}(L,1)}&=\E{\phi^\top_{x_s,\G_0} w_0w_0^\top \phi^\top_{x_{s'},\G_0}}\\
&=\phi^\top_{x_s,\G_0}\E{ w_0 w_0^\top} \phi^\top_{x_{s'},\G_0},
\end{align*}
where we use the fact that the gates $\G_0$ are statistically independent of the weights. Now let $M=\E{ w_0 w_0^\top}$, we make the following observations about $M$:

$1.$ $M(p_1,p_2)=0$, if $p_1$ and $p_2$ belong to the different bundles.

$2.$ $M(p_1,p_2)=\frac{\sigma^{2(d-1)}}{d^2_{in}}$, if $p_1$ and $p_2$ belong to the same bundle.

Using the above two observations, we have at  $t=0$:

\begin{align*}
&\E{x_s(L,1)x_{s'}(L,1)}\\&=\phi^\top_{x_s,\G_0} M \phi^\top_{x_{s'},\G_0}\\
&=\sum_{p_1,p_2=1}^{\hat{P}} \Big(x(p_1(0),s) A(x_s,p_1) \\
&\quad\quad \quad\quad \quad\quad x(p_2(0),s') A(x_{s'},p_2) M(p_1,p_2)\Big)\\
&=\frac{\sigma^{2(d-1)}}{d^2_{in}}\sum_{k=1}^{\hat{B}} \sum_{p_1,p_2\in b_k}  \Big( x(p_1(0),s) A(x_s,p_1)\\
&\quad\quad \quad\quad \quad\quad x(p_2(0),s') A(x_{s'},p_2) \Big)
\end{align*}
\end{proof}

\end{document}
